\documentclass{article}

\usepackage{microtype}
\usepackage{graphicx}
\usepackage{subfigure}
\usepackage{booktabs} 
\usepackage{amsfonts}
\usepackage{amsmath}
\usepackage{amsthm}

\usepackage{hyperref}

\usepackage{caption}

\usepackage[accepted]{icml2018}

\icmltitlerunning{Provable Defenses via the Convex Outer Adversarial Polytope}

\DeclareMathOperator*{\maximize}{maximize}
\DeclareMathOperator*{\minimize}{minimize}

\DeclareMathOperator*{\subjectto}{subject\;to}
\DeclareMathOperator*{\for}{for}

\newlength\myindent
\newtheorem*{theorem*}{Theorem}
\newtheorem{theorem}{Theorem}
\newtheorem{corollary}{Corollary}
\newtheorem{property}{Property}

\usepackage{natbib}

\begin{document}

\twocolumn[
\icmltitle{Provable Defenses against Adversarial Examples \\
           via the Convex Outer Adversarial Polytope}



\icmlsetsymbol{equal}{*}

\begin{icmlauthorlist}
\icmlauthor{Eric Wong}{mld}
\icmlauthor{J. Zico Kolter}{csd}
\end{icmlauthorlist}

\icmlaffiliation{mld}{Machine Learning Department, Carnegie Mellon University, Pittsburgh PA, 15213, USA}
\icmlaffiliation{csd}{Computer Science Department, Carnegie Mellon University, Pittsburgh PA, 15213, USA}

\icmlcorrespondingauthor{Eric Wong}{ericwong@cs.cmu.edu}
\icmlcorrespondingauthor{J. Zico Kolter}{zkolter@cs.cmu.edu}

\icmlkeywords{adversarial, robust, relu, deep network, provably, convex, adversarial polytope}

\vskip 0.3in
]



\printAffiliationsAndNotice{}  

\begin{abstract}
  We propose a method to learn deep ReLU-based classifiers that are provably
  robust against norm-bounded adversarial perturbations on the training data.
  For 
  previously unseen examples, the approach is guaranteed to detect all
  adversarial examples, though it may flag some non-adversarial examples as
  well. The basic idea is to consider \emph{a convex outer
  approximation} of the set of activations reachable through a
  norm-bounded perturbation, and we develop a robust optimization procedure that
  minimizes the worst case loss over this outer region (via a linear program).
  Crucially, we show that the dual problem to this linear program can be
  represented itself as a deep network similar to the backpropagation network,
  leading to very efficient optimization approaches that produce guaranteed
  bounds on the robust loss.
  The end result is that by executing a few more forward and backward passes
  through a slightly modified version of the original network (though possibly
  with much larger batch sizes), we can learn a classifier that is provably
  robust to \emph{any} norm-bounded adversarial attack.  We illustrate the
  approach on a number of tasks to train classifiers with robust adversarial 
  guarantees (e.g. for MNIST, we produce a convolutional classifier that provably
  has less than 5.8\% test error for any 
  adversarial attack with bounded $\ell_\infty$ norm less than $\epsilon = 0.1$),
and code for all 
experiments is available at
\url{http://github.com/locuslab/convex_adversarial}.

\end{abstract}

\section{Introduction}

Recent work in deep learning has demonstrated the prevalence of
\emph{adversarial examples}
\citep{szegedy2014intriguing,goodfellow2015explaining}, data points fed to 
a machine learning algorithm 
which are visually indistinguishable from ``normal'' examples, but which are
specifically tuned so as to fool or mislead the machine learning system.
Recent history in adversarial classification has followed something of a virtual 
``arms race'': practitioners alternatively design new ways of hardening
classifiers against existing attacks, and then a new class of
attacks is developed that can penetrate this defense.  Distillation
\citep{papernot2016distillation} was effective at preventing adversarial examples
until it was not \citep{carlini2017towards}.  There was no need to worry about
adversarial examples under ``realistic'' settings of rotation and scaling
\citep{lu2017no} until there was \citep{athalye2017synthesizing}.  
Nor does
the fact that the adversary lacks full knowledge of the model appear to be a
problem: ``black-box'' attacks are also extremely effective
\citep{papernot2017practical}.  Even detecting the presence of adversarial
examples is challenging \cite{metzen2017detecting,carlini2017adversarial}, and
attacks are not limited to  synthetic examples, having been
demonstrated repeatedly on real-world objects
\cite{sharif2016accessorize,kurakin2016adversarial}. 
Somewhat memorably, many of the adversarial defense papers 
at the most recent ICLR conference were broken prior to the review 
period completing \citep{obfuscated-gradients}. 

Given the potentially
high-stakes nature of many machine learning systems, we feel this situation is
untenable: the ``cost'' of having a classifier be fooled just once is
potentially extremely high, and so the attackers are the de-facto ``winners'' of
this current game.  Rather, one way to truly harden 
classifiers against 
adversarial attacks is to design classifiers that are \emph{guaranteed} to be robust to
adversarial perturbations, even if the attacker is given full knowledge of
the classifier.  
Any weaker attempt of ``security through obscurity''  
could ultimately prove unable to provide a robust classifier.
  

In this paper, we present a method for training \emph{provably robust}
deep ReLU classifiers, classifiers that are guaranteed to be robust against any
norm-bounded adversarial perturbations on the training set.  The approach also
provides a provable method for detecting \emph{any previously unseen} adversarial example, with
zero false negatives (i.e., the system will flag
any adversarial example in the test set, though it may also mistakenly flag some
non-adversarial examples).  The crux of our approach is to construct a
\emph{convex outer bound} on the so-called ``adversarial polytope'', the set of
all final-layer activations that can be achieved by applying a norm-bounded
perturbation to the input; if we can guarantee that the class prediction of an
example does not change within this outer bound, we have a proof that the
example could not be adversarial (because the nature of an adversarial example
is such that a small perturbation changed the class label).  We show how we can
efficiently compute and optimize over the ``worst case loss'' within this convex
outer bound, even in the case of deep networks that include relatively large 
(for verified networks) convolutional layers, and thus learn
classifiers that are provably robust to such perturbations.  From a technical
standpoint, the outer bounds we consider involve a large linear
program, but we show how to bound these optimization problems using a
formulation that computes a feasible dual solution to this linear program using
just a single backward pass through the network (and avoiding any actual 
linear programming solvers).

Using this approach we obtain, to the best of our knowledge,
by far the largest verified networks to date, with
provable guarantees of their performance under adversarial perturbations.  We
evaluate our approach on classification tasks such as human activity
recognition, MNIST 
digit classification, ``Fashion MNIST'', and street view housing
numbers.  In the case of MNIST, for example, we produce a convolutional
classifier that provably  has less than 5.8\% test error for any adversarial
attack with bounded $\ell_\infty$ norm less than $\epsilon = 0.1$.

\section{Background and Related Work}

In addition to general work in adversarial attacks and defenses, our work
relates most closely to several ongoing thrusts in adversarial examples.  First,
there is a great deal of ongoing work using exact (combinatorial) solvers to
verify properties of neural networks, including robustness to adversarial
attacks.  These typically employ either Satisfiability Modulo Theories (SMT) solvers
\cite{huang2017safety,katz2017reluplex,ehlers2017formal,carlini2017ground} or
integer programming approaches
\cite{lomuscio2017approach,tjeng2017verifying,cheng2017maximum}.  Of particular
note is the PLANET solver \cite{ehlers2017formal}, which also uses linear ReLU
relaxations, though it employs them just as a sub-step in a larger combinatorial
solver.  The obvious advantage of these approaches is that they are able to
reason about the \emph{exact} adversarial polytope, but because they are
fundamentally combinatorial in nature, it seems prohibitively difficult to scale
them even to medium-sized networks such as those we study here.  In addition,
unlike in the work we present here, the verification procedures are too
computationally costly to be integrated easily to a robust training
procedure.

The next line of related work are methods for computing \emph{tractable} bounds
on the possible perturbation regions of deep networks.  For example, Parseval
networks \cite{cisse2017parseval} attempt to achieve some degree of adversarial
robustness by 
regularizing the $\ell_2$ operator norm of the weight matrices (keeping the
network non-expansive in the $\ell_2$ norm); similarly, the work by
\citet{peck2017lower} shows 
how to limit the possible layerwise norm expansions in a variety of different
layer types.  In this work, we study similar ``layerwise'' bounds, and show that 
they are typically substantially (by many orders of magnitude) worse than the
outer bounds we present.

Finally, there is some very recent work that relates substantially to this
paper.  \citet{hein2017formal} provide provable robustness guarantees for
$\ell_2$ perturbations in two-layer networks, though they train their models
using a surrogate of their robust bound rather than the exact bound.
\citet{sinha2018certifiable} provide a method for achieving certified robustness
for perturbations defined by a certain distributional Wasserstein distance.
However, it is not clear how to translate these to traditional norm-bounded 
adversarial models (though, on the other hand, their approach also provides
generalization guarantees under proper assumptions, which is not something we
address in this paper).

By far the most similar paper to this work is the concurrent work of 
\citet{raghunathan2018certified}, who develop a semidefinite programming-based relaxation of the
adversarial polytope (also bounded via the dual, which reduces to an eigenvalue
problem), and employ this for training a robust classifier. 
However, their approach applies only to two-layer networks, and only to fully
connected networks, whereas our method applies to deep networks with arbitrary
linear operator layers such as convolution layers.  Likely due to this fact, we
are able to significantly outperform their results on medium-sized problems:
for example, whereas they attain a guaranteed robustness bound of 35\% error on
MNIST, we achieve a robust bound of 5.8\% error.  However, we also note that when we
\emph{do} use the smaller networks they consider, the bounds are complementary
(we achieve lower robust test error, but higher traditional test error); this
suggests that finding ways to combine the two bounds will be useful as a future
direction.

\begin{figure*}[t]
  \begin{center}
    \includegraphics[scale=0.4]{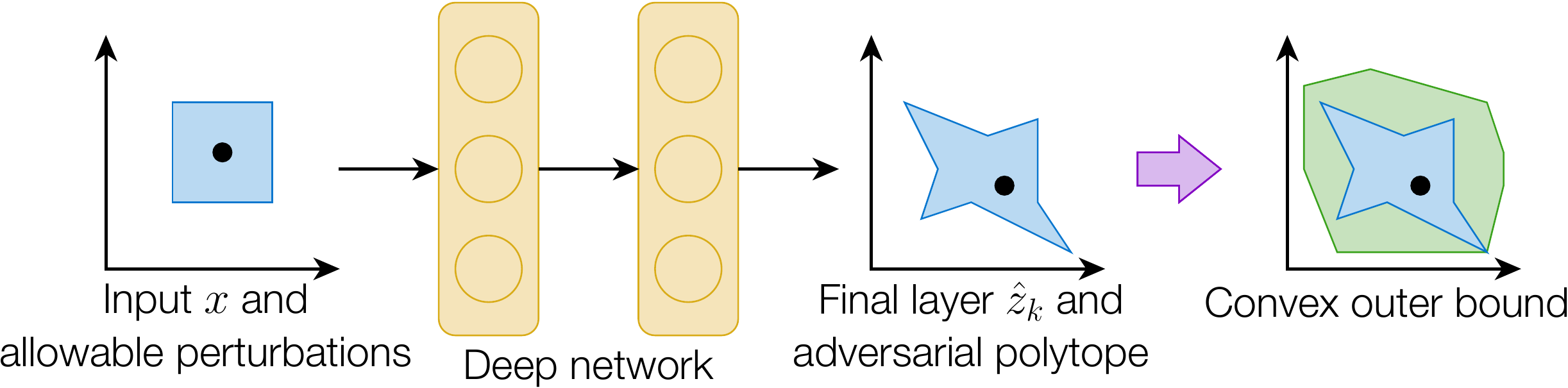}
    \vspace{-0.1in}
    \caption{Conceptual illustration of the (non-convex) adversarial
      polytope, and an outer convex bound.}
    \label{fig-outer-bound}
  \end{center}
  \vskip -0.2in
\end{figure*}

Our work also fundamentally relates to the field of robust optimization
\cite{ben2009robust}, the task of solving an optimization problem where some of
the problem data is unknown, but belong to a bounded set.  Indeed, robust optimization techniques have been used in the context of linear
machine learning models \citep{xu2009robustness} to create classifiers that are
robust to perturbations of the input.  This connection was addressed in the
original adversarial examples paper \cite{goodfellow2015explaining}, where it
was noted that for linear models,
robustness to adversarial examples can be achieved via an $\ell_1$ norm 
penalty on the weights within the loss function.\footnote{This fact is well-known in robust optimization, and we merely mean
that the original paper pointed out this connection.} \citet{madry2017towards}
revisited this connection to robust optimization, and noted that simply solving
the (non-convex) min-max formulation of the robust optimization problem works
very well in practice to find and then optimize against adversarial examples.
Our work can be seen as taking the next step in this connection between 
adversarial examples and robust optimization.  Because we consider a convex
relaxation of the adversarial polytope, we can incorporate the theory from
convex robust optimization and provide \emph{provable} bounds on the potential
adversarial error and loss of a classifier, using the specific form of
dual solutions of the optimization problem in question without relying on any
traditional optimization solver.

\section{Training Provably Robust Classifiers}

This section contains the main methodological contribution of our paper: a
method for training deep ReLU networks that are provably robust to norm-bounded
perturbations.  Our derivation roughly follows three steps: first, we define
the adversarial polytope for deep ReLU networks, and present our convex outer
bound; second, we show how we can efficiently optimize over this bound by
considering the \emph{dual problem} of the associated linear program, and
illustrate how to find solutions to this dual problem using a single 
modified backward pass 
in the original network; third, we show how to
incrementally compute the necessary elementwise upper and lower activation
bounds, using this dual approach.  After presenting this algorithm, we then
summarize how the method is applied to train provably robust classifiers, and
how it can be used to detect potential adversarial attacks on previously unseen
examples.

\subsection{Outer Bounds on the Adversarial Polytope}
In this paper we consider a $k$ layer feedforward ReLU-based neural network,
$f_\theta : \mathbb{R}^{|x|} \rightarrow \mathbb{R}^{|y|}$
given by the equations
\begin{equation}
\label{eq:nn_form}
  \begin{split}
    \hat{z}_{i+1} & = W_i z_i + b_i, \;\;\for i=1,\ldots,k-1 \\
    z_i & = \max\{\hat{z}_i, 0\}, \;\;\for i=2,\ldots,k-1
  \end{split}
\end{equation}
with $z_1 \equiv x$ and $f_\theta(x) \equiv \hat{z}_k$ (the logits input to the classifier).
We use $\theta = \{W_i, b_i\}_{i=1, \dots, k}$ to
denote the set of all parameters of the network, where $W_i$ represents a
linear operator such as matrix multiply or convolution.


We use the set $\mathcal{Z}_\epsilon(x)$ to denote the adversarial polytope, or 
the set of all final-layer activations attainable by perturbing $x$ by
 some $\Delta$ with $\ell_\infty$ norm bounded by $\epsilon$:\footnote{For the
   sake of concreteness, we will focus on the $\ell_\infty$ bound during this
   exposition, but the method does extend to other norm balls, which we will
   highlight shortly.}
\begin{equation}
  \mathcal{Z}_\epsilon(x) = \{f_\theta(x + \Delta) : \|\Delta\|_\infty \leq \epsilon\}.
\label{eq:adversarial_polytope}
\end{equation}
For multi-layer networks, $\mathcal{Z}_\epsilon(x)$ is a non-convex set (it can
be represented exactly via an integer program as in \citep{lomuscio2017approach}
or via SMT constraints \cite{katz2017reluplex}), so cannot easily be optimized over.

The foundation of our approach will be to construct a \emph{convex outer bound}
on this adversarial polytope, as illustrated in Figure \ref{fig-outer-bound}.
If no point within this outer approximation exists that will change the
class prediction of an example, then we are also guaranteed that no point within
the true adversarial polytope can change its prediction either, i.e., the point
is robust to adversarial attacks.  Our eventual approach will be to train a
network to optimize the \emph{worst} case loss over this convex outer bound,
effectively applying robust optimization techniques despite non-linearity of
the classifier.

\begin{figure}[t]
\vskip 0.05in
\begin{center}
  \includegraphics[width=3in]{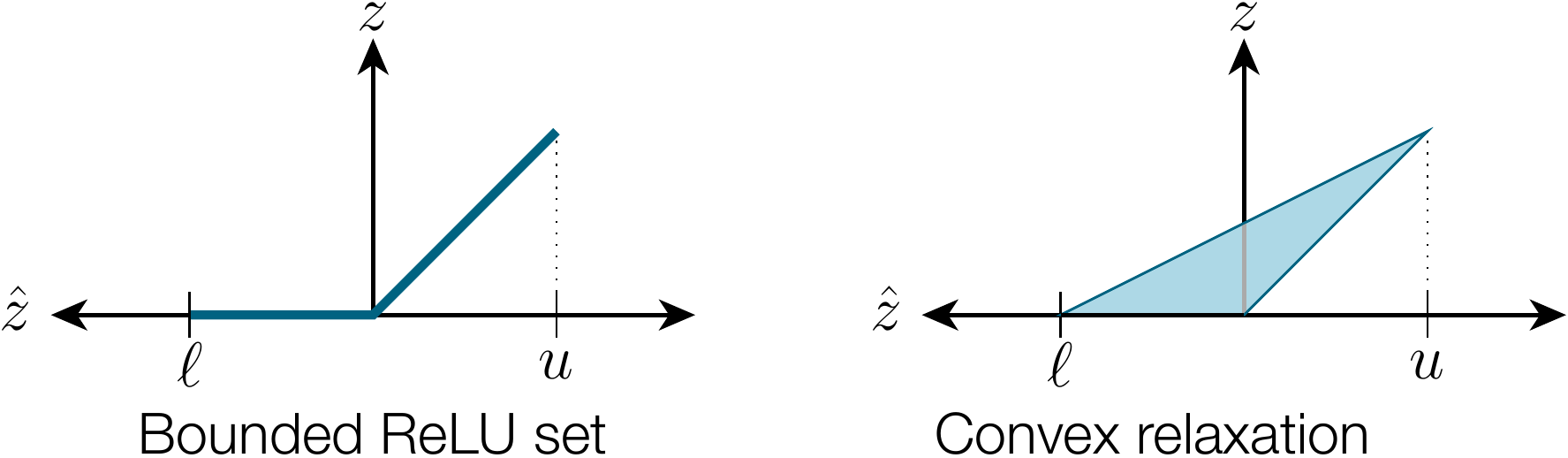}
  \vspace{-0.1in}
\caption{Illustration of the convex ReLU relaxation over the bounded set $[\ell,u]$.}
\label{fig-relu-relaxation}
\end{center}
 \vskip -0.2in
\end{figure}

The starting point of our convex outer bound is a linear relaxation of the ReLU
activations.  Specifically, given known lower and upper bounds $\ell$, $u$ for
the pre-ReLU activations, we can replace the ReLU equalities $z = \max\{0,\hat{z}\}$
from \eqref{eq:nn_form} with their upper convex envelopes,
\begin{equation}
  z \geq 0, \; z \geq \hat{z}, \; -u \hat{z} + (u-\ell) z \leq -u \ell.
\end{equation}
The procedure is illustrated in Figure \ref{fig-relu-relaxation}, and we note
that if $\ell$ and $u$ are both positive or both negative, the relaxation is
exact.  The same 
relaxation at the activation level was used in \citet{ehlers2017formal}, 
however as a sub-step for exact (combinatorial) verification of
networks, and the method for actually computing the crucial bounds
$\ell$ and $u$ is different.  We denote this outer bound on the
adversarial polytope from replacing the ReLU constraints as $\tilde{\mathcal{Z}}_\epsilon(x)$.

\paragraph{Robustness guarantees via the convex outer adversarial polytope.}  We
can use this outer bound to provide provable guarantees on the 
adversarial robustness of a classifier.  Given a sample $x$ with known label
$y^\star$, we can find the point in $\tilde{\mathcal{Z}}_\epsilon(x)$ that
minimizes this class and maximizes some
alternative target $y^{\mathrm{targ}}$, by solving the optimization problem
\begin{equation}
  \label{eq-primal-lp}
  \begin{split}
    \minimize_{\hat{z}_k} \;\; & (\hat{z}_k)_{y^\star} - (\hat{z_k})_{y^{\mathrm{targ}}} \equiv c^T \hat{z}_k\\
    \subjectto \;\; &\hat{z}_k \in \tilde{\mathcal{Z}}_\epsilon(x)
    \end{split}
\end{equation}
where $c \equiv e_{y^\star} - e_{y^{\mathrm{targ}}}$.  Importantly, this is a
\emph{linear program} (LP): the objective is linear in the 
decision variables, and our convex outer approximation consists of just linear
equalities and inequalities.\footnote{The full explicit form of this LP is given
in Appendix \ref{app:lp}.}  \emph{If we solve this LP for all target classes
$y^{\mathrm{targ}} \neq y^\star$ and find that the objective value in all cases
is \emph{positive} (i.e., we cannot make the true class activation lower than the
target even in the outer polytope), then we know that \emph{no} norm-bounded adversarial perturbation of
the input could misclassify the example.}

We can conduct similar analysis on test examples as well.  If the
network predicts some class $\hat{y}$ on an example $x$, then we can use the
same procedure as above to test whether the network will output any
\emph{different} class for a norm-bounded perturbation.  If not, then the
example \emph{cannot} be adversarial, because no input within the norm ball
takes on a different class (although of course, the network could still be
predicting the wrong class).  Although this procedure may incorrectly
``flag'' some non-adversarial examples, it will have zero false negatives, e.g.,
there may be a normal example that can still be classified differently due to a
norm-bounded perturbation, but all norm-bounded adversarial examples will be
detected.  

Of course, two major issues remain: 1) although the LP formulation can be solved
``efficiently'', actually solving an LP via traditional methods for each
example, for each target class, is not tractable; 2) we need a way of computing
the crucial $\ell$ and $u$ bounds for the linear relaxation.  We address these
in the following two sections.

\subsection{Efficient Optimization via the Dual Network}

Because solving an LP with a number of variables equal to the number of
activations in the deep network via standard approaches is not practically
feasible, the key aspect of our approach lies in our method for very efficiently
bounding these solutions.  Specifically, we consider the \emph{dual 
  problem} of the LP above; recall that any feasible dual solution provides a
guaranteed lower bound on the solution of the primal.  \emph{Crucially, we show that
the feasible set of the dual problem can itself be expressed as a deep network,
and one that is very similar to the standard backprop network}.  This means that
providing a provable lower bound on the primal LP (and hence also a provable
bound on the adversarial error), can be done with \emph{only a single backward
  pass through a slightly modified network} (assuming for the time being, that
we still have known upper and lower bounds for each activation).  This is
expressed in the following theorem
\begin{theorem}
\label{thm:dual}
The dual of \eqref{eq-primal-lp} is of the form 
\begin{equation}
  \begin{split}
    \maximize_{\alpha} \;\; & J_\epsilon(x,g_\theta(c, \alpha))  \\
    \subjectto \;\; & \alpha_{i,j} \in [0,1], \; \forall i,j
    \end{split}
\end{equation}
where $J_\epsilon(x,\nu)$ is equal to 
\begin{equation}
    -\sum_{i=1}^{k-1} \nu_{i+1}^T b_i - x^T\hat{\nu}_1 - \epsilon \|\hat{\nu}_1\|_1 + 
    \sum_{i=2}^{k-1}\sum_{j \in \mathcal{I}_i} \ell_{i,j} [\nu_{i,j}]_+
    \label{eq:J}
\end{equation}
and $g_\theta(c,\alpha)$ is a $k$ layer feedforward neural network given by the equations
\begin{equation}
\label{eq:dual_network}
  \begin{split}
    \nu_k & = -c \\
    \hat{\nu}_i & = W_i^T \nu_{i+1}, \;\; \for i=k-1,\ldots,1 \\
    \nu_{i,j} & = \left \{
      \begin{array}{ll}
        0 & j \in \mathcal{I}^{-}_i \\
        \hat{\nu}_{i,j} & j \in \mathcal{I}^{+}_i \\
        \frac{u_{i,j}}{u_{i,j} - \ell_{i,j}} [\hat{\nu}_{i,j}]_+ - \alpha_{i,j} [\hat{\nu}_{i,j}]_- & j \in \mathcal{I}_i, \\
      \end{array} \right .\\
    &\kern 1.5in \for i= k-1, \dots, 2 \\
  \end{split}
\end{equation}
where $\nu$ is shorthand for $(\nu_i,
\hat\nu_i)$ for all $i$ (needed because the objective $J$ depends on \emph{all}
$\nu$ terms, not just the first), and where $\mathcal{I}_i^{-}$,
$\mathcal{I}_i^{+}$, and $\mathcal{I}_i$ denote the 
sets of activations in layer $i$ where the lower and upper bounds are both
negative, both positive, or span zero respectively. 
\end{theorem}
The ``dual network'' from \eqref{eq:dual_network} in fact is almost 
identical to the backpropagation network, except that for nodes $j$ in
$\mathcal{I}_i$ there is the additional free variable $\alpha_{i,j}$ that we can
optimize over to improve the objective.  In practice, rather than optimizing
explicitly over $\alpha$, we choose the fixed, dual feasible solution
\begin{equation}
  \alpha_{i,j} = \frac{u_{i,j}}{u_{i,j} - \ell_{i,j}}.
\end{equation}
This makes the entire backward pass a \emph{linear} function, and is additionally
justified by considerations regarding the conjugate set of the ReLU relaxation
(see Appendix \ref{app:alpha} for discussion).  Because \emph{any} solution
$\alpha$ is still dual feasible, this still provides a lower bound on the primal
objective, and one that is reasonably tight in practice.\footnote{The tightness of the bound is examined in Appendix \ref{app:experiments}.}  Thus, in the remainder
of this work we simply refer to the dual objective as $J(x, g_\theta(c))$, implicitly
using the above-defined $\alpha$ terms.

We also note that norm bounds other than the $\ell_\infty$ norm are also
possible in this framework: if the input perturbation is bounded within some
convex $\ell_p$ norm, then the only difference in the dual formulation is that
the $\ell_1$ norm on $\|\hat{\nu}\|_1$ changes to $\|\hat{\nu}\|_q$ where $q$ is
the dual norm of $p$.  However, because we focus solely on experiments with the
$\ell_\infty$ norm below, we don't emphasize this point in the current paper.

\subsection{Computing Activation Bounds}

\label{sec:compute_bounds}
\begin{algorithm}[tb]
   \caption{Computing Activation Bounds}
   \label{alg:bounds}
\begin{algorithmic}
  \STATE \textbf{input:} Network parameters $\{W_i,b_i\}_{i=1}^{k-1}$, data point
  $x$, ball size $\epsilon$
  \STATE \emph{// initialization}
  \STATE $\hat{\nu}_1 := W_1^T$
  \STATE $\gamma_1 := b_1^T$
  \STATE $\ell_2 := x^T W_1^T + b_1^T - \epsilon\|W_1^T\|_{1,:} $ 
  \STATE $ u_2 := x^T W_1^T + b_1^T + \epsilon\|W_1^T\|_{1,:} $
  \STATE \emph{//
    $\|\cdot\|_{1,:}$ for a matrix here denotes $\ell_1$ norm of all columns}
  \FOR{$i=2,\ldots,k-1$}
  \STATE form $\mathcal{I}_i^{-}$, $\mathcal{I}_i^+$, $\mathcal{I}_i$; form $D_i$ as
  in \eqref{eq:d_after_alpha}
  \STATE \emph{// initialize new terms}
  \STATE $\nu_{i,\mathcal{I}_i} := (D_i)_{\mathcal{I}_i} W_i^T $
  \STATE $\gamma_i := b_i^T$
  \STATE \emph{// propagate existing terms}
  \STATE $\nu_{j,\mathcal{I}_j} := \nu_{j,\mathcal{I}_j} D_i W_i^T, \;\;
  j=2,\ldots,i-1 $
  \STATE $\gamma_j :=  \gamma_j D_i W_i^T, \;\; j=1,\ldots,i-1 $
  \STATE $\hat{\nu}_1 := \hat{\nu}_1 D_i W_i^T$
  \STATE \emph{// compute bounds}
  \STATE $\psi_i := x^T \hat{\nu}_1 + \sum_{j=1}^i \gamma_j$
  \STATE $\ell_{i+1} := \psi_i -
  \epsilon\|\hat{\nu}_1\|_{1,:}  + \sum_{j=2}^{i}  \sum_{i' \in \mathcal I_i} \ell_{j,i'}  [-\nu_{j,i'}]_+$
  \STATE $u_{i+1} := \psi_i + \epsilon\|\hat{\nu}_1\|_{1,:}  - \sum_{j=2}^{i} \sum_{i' \in \mathcal I_i} \ell_{j,i'}  [\nu_{j,i'}]_+$
  \ENDFOR
  \STATE \textbf{output:} bounds $\{\ell_i,u_i\}_{i=2}^{k}$
\end{algorithmic}
\end{algorithm}

Thus far, we have ignored the (critical) issue of how we actually obtain the
elementwise lower and upper bounds on the pre-ReLU activations, $\ell$ and $u$.
Intuitively, if these bounds are too loose, then the adversary has too much
``freedom'' in crafting adversarial activations in the later layers that don't
correspond to any actual input.  However, because the dual function
$J_\epsilon(x,g_\theta(c))$ provides a bound on \emph{any} linear function $c^T
\hat{z}_k$ of the 
final-layer coefficients, we can compute $J$ for $c=I$ and $c=-I$ 
to obtain lower and upper bounds on these coefficients.  For $c = I$,
the backward pass variables (where $\hat{\nu}_i$ is now a matrix) are given by
\begin{equation}
  \label{eq:nu-approx}
  \begin{split}
  \hat{\nu}_i & = -W_i^T D_{i+1} W_{i+1}^T \ldots D_{n} W_n^T \\
  \nu_i & = D_i \hat{\nu}_i
  \end{split}
\end{equation}
where $D_i$ is a diagonal matrix with entries 
\begin{equation}
\label{eq:d_after_alpha}
  \begin{split}
	(D_i)_{jj} & = \left \{
      \begin{array}{ll}
        0 & j \in \mathcal{I}^{-}_i \\
        1 & j \in \mathcal{I}^{+}_i \\
        \frac{u_{i,j}}{u_{i,j} - \ell_{i,j}} & j \in \mathcal{I}_i \\
      \end{array} \right .
   \end{split} .
 \end{equation}

We can compute $(\nu_i, \hat\nu_i)$ and the corresponding upper bound $J_\epsilon(x, \nu)$ (which is now a vector) in
a layer-by-layer fashion, first generating bounds on $\hat{z}_2$, then using
these to generate bounds on $\hat{z}_3$, etc.  

The resulting algorithm, which uses these backward pass variables in matrix form
to incrementally build the bounds, is described in Algorithm
\ref{alg:bounds}.  From here on, the computation of $J$ will 
implicitly assume that we also compute the bounds. 
 Because the full algorithm is somewhat involved, we
highlight that there are two dominating costs to the full bound computation: 1)
computing a forward pass through the network on an ``identity matrix'' (i.e.,
a basis vector $e_i$ for each dimension $i$ of the input); and 2) computing a
forward pass starting at an intermediate layer, once for each activation in the set
$\mathcal{I}_i$ (i.e., for each activation where the upper and lower bounds span
zero).  Direct computation of the bounds requires computing these forward passes
explicitly, since they ultimately factor into the nonlinear terms in the $J$
objective, and this is admittedly the poorest-scaling aspect of our approach.  A
number of approaches to scale this to larger-sized inputs is possible,
including bottleneck layers earlier in the network, e.g. PCA processing of the
images, random projections, or other similar constructs; at the current point,
however, this remains as future work.  Even without improving scalability, the
technique already can be applied to much larger networks than any alternative
method to prove robustness in deep networks that we are aware of.

\subsection{Efficient Robust Optimization}
\label{sec:robust}
Using the lower bounds developed in the previous sections, we can develop an
efficient optimization approach to training provably robust deep networks.
Given a data set $(x_i, y_i)_{i=1,\dots,N}$, instead of minimizing the loss at these data
points, we minimize (our bound on) the \emph{worst} location (i.e. with the
highest loss) in an $\epsilon$ ball around each $x_i$, i.e.,
\begin{equation}
  \label{eq:robust-opt}
  \minimize_\theta \;\; \sum_{i=1}^N \max_{\|\Delta\|_\infty \leq \epsilon} L(f_\theta(x_i + \Delta), y_i).
\end{equation}
This is a standard robust optimization objective, but prior to this work it was
not known how to train these classifiers when $f$ is a deep nonlinear network.

We also require that a multi-class loss function have the following property (all of
cross-entropy, hinge loss, and zero-one loss have this property):
\begin{property}
\label{prop:invariance}
A multi-class loss function $L : \mathbb{R}^{|y|} \times \mathbb{R}^{|y|} \rightarrow
\mathbb{R}$ is translationally invariant if for all $a\in\mathbb{R}$, 
\begin{equation}
\label{eq:invariance}
L(y,y^\star) = L(y-a1,y^\star).
\end{equation}
\end{property}
Under this assumption, we can upper bound the robust optimization problem using
our dual problem in Theorem \ref{thm:robust_opt}, which we prove in Appendix
\ref{app:upper}.  

\begin{theorem}
Let $L$ be a monotonic loss function that satisfies Property \ref{prop:invariance}. 
For any data point $(x,y)$, and $\epsilon>0$, the worst case adversarial loss
from  \eqref{eq:robust-opt} can be upper bounded by
\begin{equation}
\max_{\|\Delta\|_\infty \leq \epsilon} L(f_\theta(x + \Delta), y) \leq
L(-J_{\epsilon}(x,g_\theta(e_{y} 1^T - I)), y),
\end{equation}
where $J_{\epsilon}$ is vector valued and 
as defined in \eqref{eq:J} for a given $\epsilon$, and
$g_\theta$ is as defined in \eqref{eq:dual_network} for the given model
parameters $\theta$.  
\label{thm:robust_opt}
\end{theorem}
We denote the upper bound from Theorem \ref{thm:robust_opt} as the robust loss.
Replacing the summand of \eqref{eq:robust-opt} with the robust loss results 
in the following minimization problem
\begin{equation}
\minimize_{\theta} \;\; \sum_{i=1}^N L(-J_{\epsilon}(x_i, g_\theta(e_{y_i} 1^T - I)), y_i).
\end{equation}
All the network terms,
including the upper and lower bound computation, are differentiable, so the
whole optimization can be solved with any standard stochastic gradient variant and autodiff toolkit,
and the result is a network that (if we achieve low loss) is guaranteed to be robust
to adversarial examples.

\subsection{Adversarial Guarantees}

Although we previously described, informally, the guarantees provided by our
bound, we now state them formally.
The bound for the robust optimization procedure gives rise to several
\emph{provable} metrics measuring robustness  and detection of adversarial
attacks, which can be computed for any ReLU based neural network  
independently from how the network was trained; however, not surprisingly, the
bounds are by far the tightest and the most useful in cases where the network was
trained explicitly to minimize a robust loss.

\paragraph{Robust error bounds}
The upper bound from Theorem \ref{thm:robust_opt} functions as 
a certificate that guarantees robustness around an example (if 
classified correctly), as described in Corollary \ref{thm:guarantee}. The proof
is immediate, but included in Appendix \ref{app:guarantee}.

\begin{corollary}
\label{thm:guarantee}
For a data point $x$, label $y^\star$ and $\epsilon>0$, if 
\begin{equation}
J_{\epsilon}(x, g_\theta(e_{y^\star} 1^T - I)) \geq 0
\end{equation} 
(this quantity is a vector, so the inequality means that all elements must be
greater than zero) then the model is guaranteed to be robust around this data
point. Specifically, there does not exist an adversarial example $\tilde x$ 
such that $\|\tilde x - x\|_\infty \leq \epsilon$ and $f_\theta(\tilde x) \neq y^\star$. 
\end{corollary}
We denote the fraction of examples that do not have this certificate as the
robust error. Since adversaries can only hope to attack examples without this 
certificate, the robust error is a provable upper bound on the achievable 
error by \emph{any} adversarial attack.

\paragraph{Detecting adversarial examples at test time}
The certificate from Theorem \ref{thm:guarantee} can also be modified trivially
to detect adversarial examples at test time.
Specifically, we replace the bound based upon the true class $y^\star$ to a
bound based upon just the predicted class $\hat{y} = \max_y f_\theta(x)_y$.  In
this case we have the following simple corollary.

\begin{corollary}
\label{thm:detection}
For a data point $x$, model prediction $\hat{y} = \max_y f_\theta(x)_y$ and $\epsilon>0$, if 
\begin{equation}
J_{\epsilon}(x, g_\theta(e_{\hat{y}} 1^T - I)) \geq 0
\end{equation} 
then $x$ cannot be an adversarial example.  Specifically, $x$ cannot be a
perturbation of a ``true'' example $x^\star$ with $\|x-x^\star\|_\infty \leq
\epsilon$, such that the model would correctly classify $x^\star$, but
incorrectly classify $x$.
\end{corollary}

This corollary follows immediately from the fact that the robust bound
guarantees no example with $\ell_\infty$ norm within $\epsilon$ of $x$
is classified differently from $x$.  This approach may 
classify non-adversarial inputs as 
potentially adversarial, but it has zero false negatives, in that it will never
fail to flag an adversarial example.  Given the challenge in even defining
adversarial examples in general, this seems to be as strong a guarantee as is
currently possible.

\paragraph{$\epsilon$-distances to decision boundary}
Finally, for each example $x$ on a fixed network, we can compute  the
largest value of $\epsilon$ for which a certificate of robustness 
exists, i.e., such that the output $f_\theta(x)$ provably cannot be flipped
within the $\epsilon$ ball.  Such an epsilon gives a lower bound on the
$\ell_\infty$ distance from the example to the decision boundary (note that the
classifier may or may not  actually be correct). Specifically, if we find
$\epsilon$ to solve the optimization problem
\begin{equation}
  \label{eq:eps_dist}
  \begin{split}
    \maximize_{\epsilon} \;\; & \epsilon\\
    \subjectto \;\; & J_{\epsilon}(x, g_\theta(e_{f_\theta(x)} 1^T - I))_y \geq 0,
  \end{split}
\end{equation}
then we know that $x$ must be at least $\epsilon$ away from the decision boundary in $\ell_\infty$ 
distance, and that this is the largest $\epsilon$ for which we have a certificate of robustness. 
The certificate is monotone in $\epsilon$, 
and the problem can be solved using Newton's method.

\section{Experiments}

Here we demonstrate the approach on small and medium-scale problems.  Although the method does not yet scale to ImageNet-sized classifiers,
we do demonstrate the approach on a simple convolutional network applied to
several image classification problems, illustrating that the method can apply to 
approaches beyond very small 
fully-connected networks (which represent the state of the
art for most existing work on neural network verification).  Scaling challenges
were discussed briefly above, and we highlight them more below.   Code for
these experiments 
is available at
\url{http://github.com/locuslab/convex_adversarial}. 

A summary of all the experiments is in Table \ref{tab:results}. For all
experiments, we report the clean test error, the error achieved by
 the fast gradient sign method \cite{goodfellow2015explaining}, the
error achieved by the projected gradient descent approach \cite{madry2017towards},
and the robust error bound. In all cases, the robust error
bound for the robust model is significantly lower than
the achievable error rates by PGD under standard training. All 
experiments were run on a single Titan X GPU. For more 
experimental details, see Appendix \ref{app:experiments}. 


\begin{table*}[t]
\caption{Error rates for various problems and attacks, and our robust bound for baseline and robust models.}
\label{tab:results}
\vskip 0.15in
\begin{center}
\begin{small}
\begin{sc}
\begin{tabular}{lccccccc}
\toprule
Problem & Robust & $\epsilon$ & Test error & FGSM error & PGD error & Robust error bound \\
\midrule
MNIST     & $\times$ & 0.1 & 1.07\% & 50.01\% & 81.68\% & 100\% \\
MNIST     & $\surd$ & 0.1 & 1.80\% & 3.93\% & 4.11\% & 5.82\% \\
\midrule
Fashion-MNIST     & $\times$ & 0.1 & 9.36\% & 77.98\% & 81.85\% & 100\% \\
Fashion-MNIST     & $\surd$ & 0.1 & 21.73\% & 31.25\% & 31.63\% & 34.53\% \\
\midrule
HAR & $\times$ & 0.05 & 4.95\% & 60.57\% & 63.82\% & 81.56\% \\
HAR & $\surd$  & 0.05 & 7.80\% & 21.49\% & 21.52\% & 21.90\% \\
\midrule
SVHN & $\times$ & 0.01 & 16.01\% & 62.21\% & 83.43\% & 100\% \\
SVHN & $\surd$ & 0.01 & 20.38\% & 33.28\% & 33.74\% & 40.67\% \\
\bottomrule
\end{tabular}
\end{sc}
\end{small}
\end{center}
\vskip -0.1in
\end{table*}

\subsection{2D Example} 
We consider training a robust binary classifier
on a 2D input space with randomly generated spread out data points. 
Specifically,
we use a 2-100-100-100-100-2 fully connected network.
Note that
there is no notion of generalization here; we are just visualizing and 
evaluating the ability of the learning approach to fit a classification 
function robustly.

\begin{figure}[t]
\begin{center}
  \centerline{\includegraphics[width=2.5in]{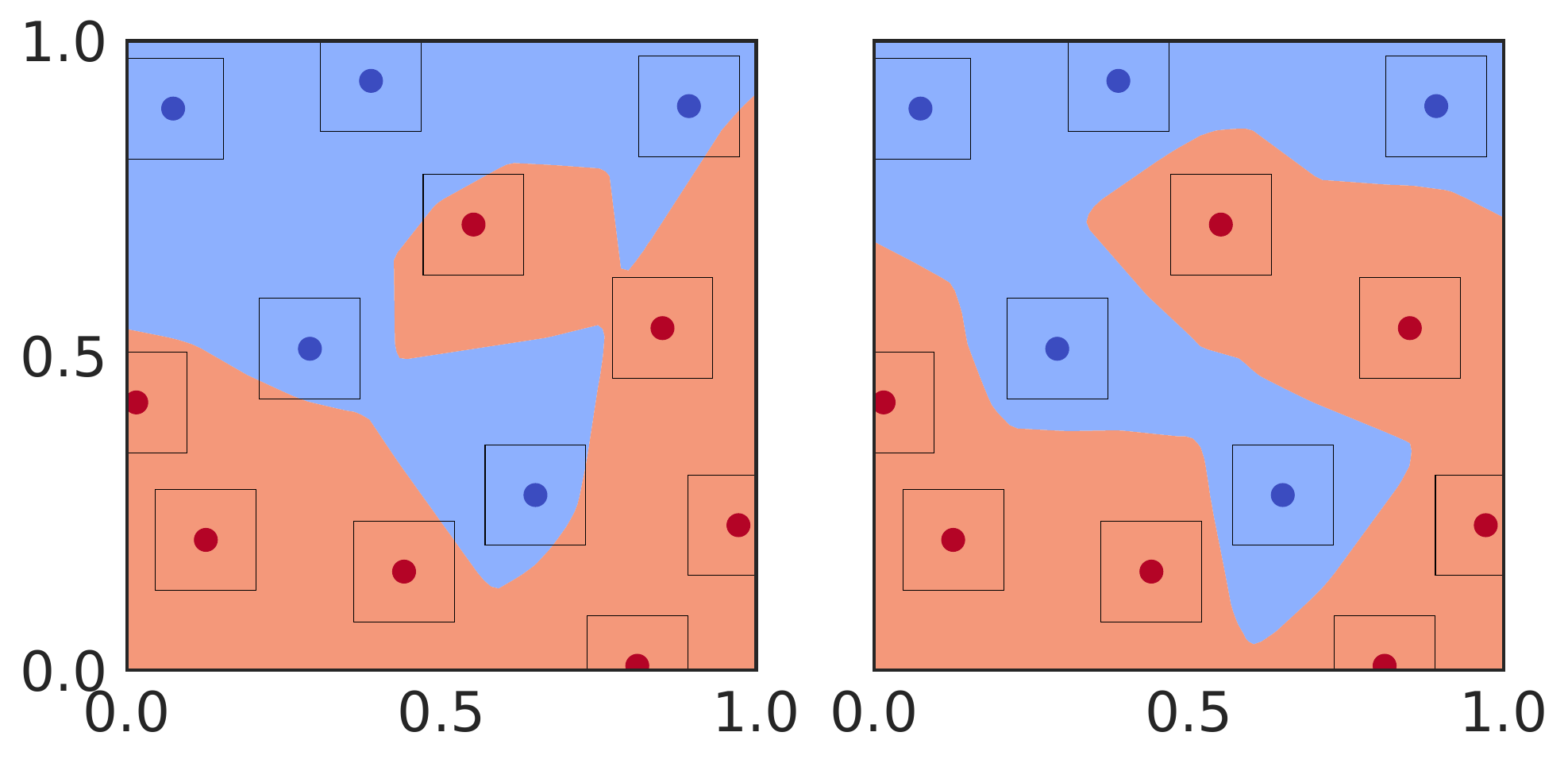}}
  \vspace{-0.1in}
  \caption{Illustration of classification boundaries resulting from standard
    training (left) and robust training (right) with $\ell_\infty$ balls of size
    $\epsilon=0.08$ (shown in figure). }
  \label{fig-robust}
\end{center}
\vskip -0.4in
\end{figure}

Figure \ref{fig-robust} shows the resulting classifiers produced by standard
training (left) and robust training via our method (right).  As expected, the 
standard training approach results in points
that are classified differently somewhere within their $\ell_\infty$ ball of
radius $\epsilon=0.08$ (this is exactly an adversarial example for the training
set). In contrast, 
the robust training method is able to attain zero robust error and 
provides a classifier that is
guaranteed to classify all points within the balls correctly. 



\subsection{MNIST}
We present results on a provably robust classifier on the
MNIST data set.  Specifically, we consider a ConvNet architecture 
that includes two convolutional layers, with 16 and 32
channels  (each with a stride of two, to decrease the resolution by half without
requiring max pooling layers), and two fully connected layers stepping down to
100 and then 10 (the output dimension) hidden units, with ReLUs following each
layer except the last.  

\begin{figure}[t]
\begin{center}
  \centerline{\includegraphics[width=\columnwidth]{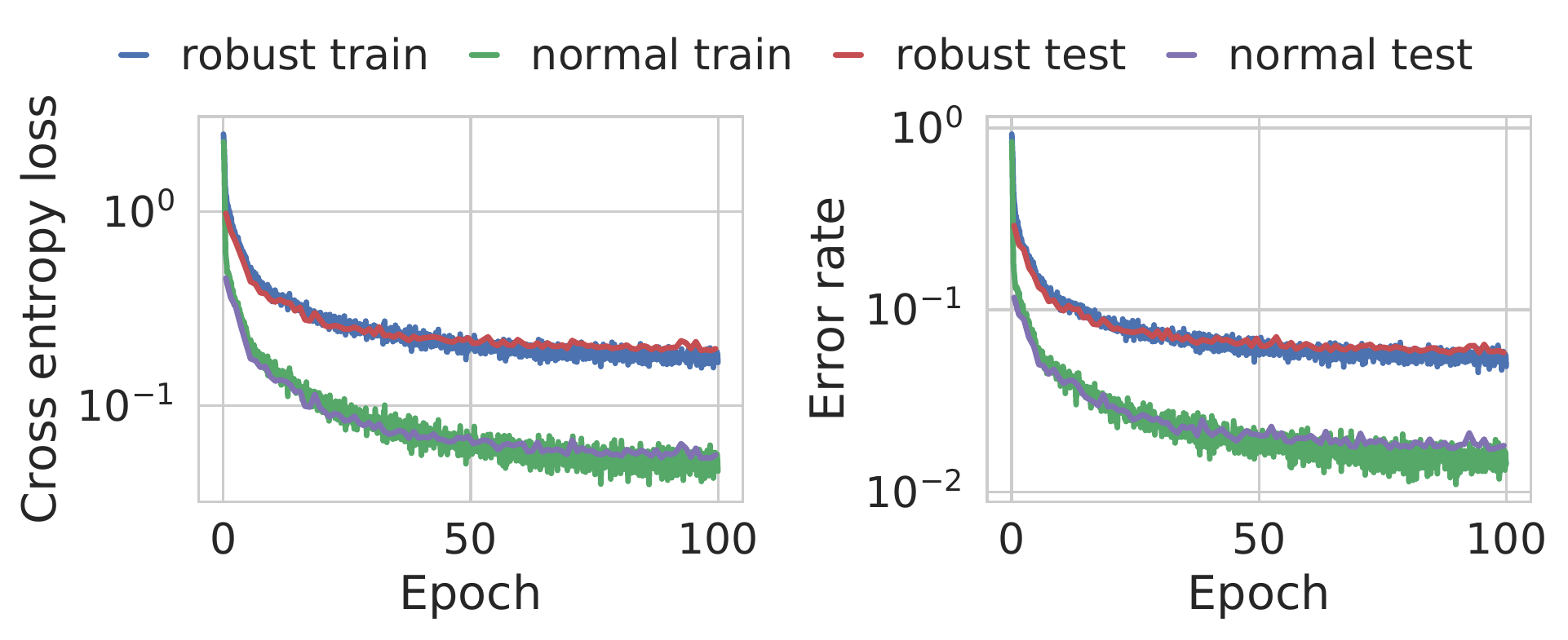}}
  \vspace{-0.1in}
\caption{Loss (left) and error rate (right) when training a robust convolutional
  network on the MNIST dataset. Similar learning curves for the other
  experiments can be found in Appendix \ref{app:experiments}.} 
\label{fig:mnist_curves}
\end{center}
\vskip -0.2in
\end{figure}

Figure \ref{fig:mnist_curves} shows the training progress using our procedure
with a robust softmax loss function and $\epsilon=0.1$.  As described in 
Section \ref{sec:robust}, any norm-bounded adversarial technique will be unable
to achieve loss or error higher than the robust bound.  
The final classifier after 100 epochs reaches a 
test error of 1.80\% with a robust test error of  5.82\%.  For a
traditionally-trained classifier (with 1.07\% test error) the FGSM approach
results in 50.01\% error, while PGD results in 81.68\% error.  On the classifier
trained with our method, however, FGSM
and PGD only achieve errors of 3.93\% and 4.11\% respectively (both, naturally,
below our bound of 5.82\%).  These results are summarized in Table
\ref{tab:results}.



\begin{figure}[t]
\begin{center}
  \centerline{\includegraphics[width=\columnwidth]{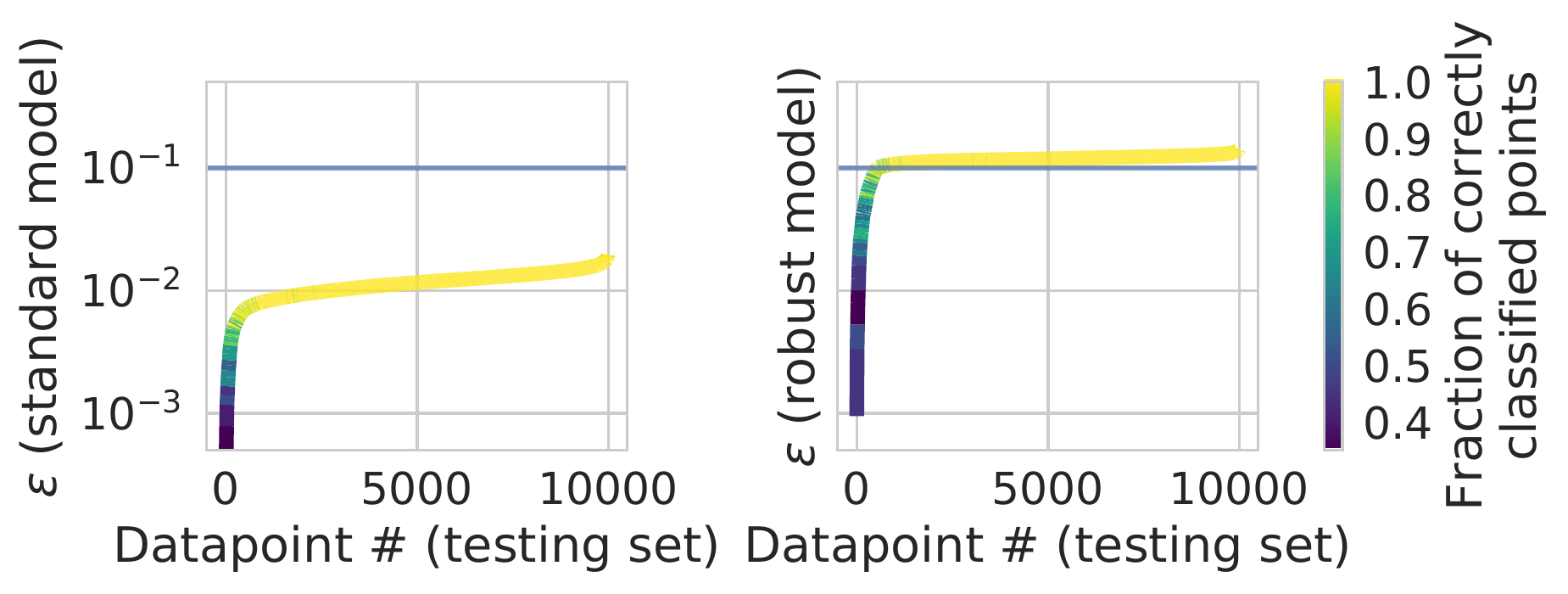}}
  \vspace{-0.1in}
\caption{Maximum $\epsilon$ distances to the decision boundary of each data point 
in increasing $\epsilon$ order for standard and robust models (trained 
with $\epsilon=0.1$). 
The color encodes the fraction of points which 
were correctly classified. }
\label{fig:mnist_epsilons}
\end{center}
\vspace{-0.3in}
\end{figure}

\paragraph{Maximum $\epsilon$-distances}
Using Newton's method with backtracking line search, for each example, 
we can compute in 5-6 Newton steps the maximum $\epsilon$ that is robust as
described in \eqref{eq:eps_dist} for both a standard classifier and the robust
classifier.  
Figure \ref{fig:mnist_epsilons} shows the maximum $\epsilon$ values calculated for each 
testing data point under standard training and robust training. Under standard training, 
the correctly classified examples have a lower bound of  around $0.007$ away from the decision boundary. 
However, with robust training this value is pushed to 0.1, which is expected since that is the 
robustness level used to train the model. We also observe that the incorrectly
classified examples all tend to be relatively closer to the decision boundary.

\subsection{Other Experiments}
\paragraph{Fashion-MNIST}
We present the results of our robust classifier on the Fashion-MNIST 
dataset \citep{xiao2017fashion}, 
a harder dataset with the same size (in dimension and number of examples) as
MNIST (for which input binarization is a reasonable defense). Using the 
same architecture as in MNIST, for $\epsilon=0.1$, we 
achieve a robust error of 34.53\%, which is fairly close to the PGD
error rate of 31.63\% (Table \ref{tab:results}).  Further
experimental details are in Appendix \ref{app:fashion}.

\paragraph{HAR}
We present results on a human activity recognition dataset \citep{anguita2013public}. 
Specifically, we consider a fully connected network with one layer of 500 hidden units
and $\epsilon=0.05$, 
achieving 21.90\% robust error. 


\paragraph{SVHN}
Finally, we present results on SVHN. The goal here is not to achieve state of the 
art performance on SVHN, but to create a deep convolutional classifier
for real world images with 
provable guarantees. Using the same architecture as in MNIST, for
$\epsilon=0.01$ we achieve a robust error bound of 42.09\%, with PGD 
achieving 34.52\% error. 
Further experimental details are in Appendix \ref{app:svhn}.

\subsection{Discussion}
Although these results are relatively small-scale, the somewhat surprising
ability here is that by just considering a few more forward/backward passes in a
modified network to compute an alternative loss, we can derive \emph{guaranteed}
error bounds for any adversarial attack.  
While this is by no means state of the art performance on standard benchmarks,
this is by far the largest provably verified network we are currently aware of,
and 5.8\% robust error on MNIST represents reasonable performance given that it
is against \emph{any} adversarial attack strategy bounded in $\ell_\infty$ norm, in comparison to the only
other robust bound of 35\% from \citet{raghunathan2018certified}. 

Scaling to ImageNet-sized classification problems remains a challenging task; 
the MNIST classifier takes about 5 hours to train for 100 epochs on a single Titan X GPU,
which is between two and three orders of magnitude more costly than naive training. 
But because the approach is not
combinatorially more expensive in its complexity, we believe it represents
a much more 
feasible approach than those based upon integer programming or satisfiability,
which seem highly unlikely to ever scale to such problems.  Thus, we believe the
current performance represents a substantial step forward in research on
adversarial examples.

\section{Conclusion}

In this paper, we have presented a method based upon
linear programming and duality theory for training classifiers that are provably
robust to norm-bounded adversarial attacks.   Crucially, instead of solving anything costly, 
we design an objective equivalent to a few passes through the
original network (with larger batch size), 
that is a guaranteed bound on the robust error and loss of the
classifier.  

While we feel this is a substantial step forward in 
defending classifiers, 
two main directions
for improvement exist, the first of which is scalability. 
Computing the bounds requires sending an identity matrix through the network, 
which amounts to a sample for every \emph{dimension} of the input vector 
(and more at
intermediate layers, for each activation with bounds that span
zero).  For domains like ImageNet, this is completely infeasible, and techniques
such as using bottleneck layers, other dual bounds, and
random projections are likely necessary.
However, unlike many past approaches, this scaling is not fundamentally
combinatorial, so has some chance of success even in large networks.

Second, it will be necessary to characterize attacks beyond simple
norm bounds. While $\ell_\infty$ bounded examples 
offer a 
compelling visualization of images that look ``identical'' to existing
examples, this is by no means the only set of possible
attacks.  For example, the work in \citet{sharif2016accessorize} was able
to break face  recognition software by using manufactured glasses, which is
clearly not bounded in $\ell_\infty$ norm, 
and the work in \citet{engstrom2017rotation} was able to fool convolutional networks with 
simple rotations and translations. 
Thus, a great deal of work remains to understand
both the space of adversarial examples that we \emph{want} classifiers to be
robust to, as well as methods for dealing with these likely highly non-convex
sets in the \emph{input} space.

Finally, although our focus in this paper was on adversarial examples and robust
classification, the general techniques described here (optimizing over
relaxed convex networks, and using a non-convex network representation of the
dual problem to derive guaranteed bounds), may find applicability well beyond
adversarial examples in deep learning.  Many problems that
invert neural networks or optimize over latent
spaces involve optimization problems that are a function of the neural network
inputs or activations, and similar techniques may be brought to bear in these
domains as well.

%

\section*{Acknowledgements}
This work was supported by a DARPA Young Faculty Award, under grant number N66001-17-1-4036. 
We thank Frank R. Schmidt for providing helpful comments on an earlier draft of this work. 
%
%

\bibliography{convex_adversarial}
\bibliographystyle{icml2018}

\clearpage
\appendix
\section{Adversarial Polytope}
\subsection{LP Formulation}
\label{app:lp}
Recall 
  \eqref{eq-primal-lp}, which uses a convex
outer bound of the adversarial polytope. 
\begin{equation}
  \minimize_{\hat{z}_k} c^T \hat{z_k}, \;\; \subjectto \hat{z}_k \in \tilde{\mathcal{Z}}_\epsilon(x)
\end{equation}
With the convex outer bound on the ReLU constraint 
and the adversarial perturbation on the input, this minimization problem is the 
following linear program 
\begin{equation}
\label{eq:lp}
  \begin{split}
  &\quad\minimize_{\hat{z}_k} c^T\hat{z_k}, \;\; \subjectto\\
    & \hat{z}_{i+1} = W_i z_i + b_i, \; i=1,\ldots,k-1 \\
  &z_1 \leq x + \epsilon \\
  &z_1 \geq x - \epsilon \\
  &z_{i,j} = 0, \; i=2,\ldots,k-1, j \in \mathcal{I}^{-}_i \\
  &z_{i,j}  = \hat{z}_{i,j}, \; i=2,\ldots,k-1, j \in \mathcal{I}^{+}_i \\
  & \!\! \left . \begin{aligned}
    & z_{i,j}  \geq 0, \\
    & z_{i,j} \geq \hat{z}_{i,j},  \\
    & \bigg((u_{i,j}-\ell_{i,j}) z_{i,j}\big. \\
    & \big.-u_{i,j} \hat{z}_{i,j}\bigg) \leq -u_{i,j} \ell_{i,j}
  \end{aligned} \; \right \} \; i=2,\ldots,k-1, j \in \mathcal{I}_i
  \end{split}
\end{equation}

\subsection{Proof of Theorem \ref{thm:dual}}
\label{app:dual}
In this section we derive the dual of the LP in \eqref{eq:lp}, 
in order to prove Theorem \ref{thm:dual}, reproduced below:
\begin{theorem*}
The dual of \eqref{eq-primal-lp} is of the form 
\begin{equation}
  \begin{split}
    \maximize_{\alpha} \;\; & J_\epsilon(x,g_\theta(c, \alpha))  \\
    \subjectto \;\; & \alpha_{i,j} \in [0,1], \; \forall i,j
    \end{split}
\end{equation}
where $J_\epsilon(x,\nu) = $
\begin{equation}
   -\sum_{i=1}^{k-1} \nu_{i+1}^T b_i - x^T\hat{\nu}_1 - \epsilon \|\hat{\nu}_1\|_1 + 
    \sum_{i=2}^{k-1}\sum_{j \in \mathcal{I}_i} \ell_{i,j} [\nu_{i,j}]_+
\end{equation}
and $g_\theta(c,\alpha)$ is a $k$ layer feedforward neural network given by the equations
\begin{equation}
  \begin{split}
    \nu_k & = -c \\
    \hat{\nu}_i & = W_i^T \nu_{i+1}, \;\; \for i=k-1,\ldots,1 \\
    \nu_{i,j} & = \left \{
      \begin{array}{ll}
        0 & j \in \mathcal{I}^{-}_i \\
        \hat{\nu}_{i,j} & j \in \mathcal{I}^{+}_i \\
        \frac{u_{i,j}}{u_{i,j} - \ell_{i,j}} [\hat{\nu}_{i,j}]_+ - \alpha_{i,j} [\hat{\nu}_{i,j}]_- & j \in \mathcal{I}_i, \\
      \end{array} \right .\\
    &\kern 1.5in \for i= k-1, \dots, 2 \\
  \end{split}
\end{equation}
where $\nu$ is shorthand for $(\nu_i,
\hat\nu_i)$ for all $i$ (needed because the objective $J$ depends on \emph{all}
$\nu$ terms, not just the first), and where $\mathcal{I}_i^{-}$,
$\mathcal{I}_i^{+}$, and $\mathcal{I}_i$ denote the 
sets of activations in layer $i$ where the lower and upper bounds are both
negative, both positive, or span zero respectively. 
\end{theorem*}
\begin{proof}
In detail, we associate the following dual variables with each of the constraints
\begin{equation}
  \begin{split}
    \hat{z}_{i+1} = W_i z_i + b_i \Rightarrow \nu_{i+1} &\in \mathbb{R}^{|\hat{z}_{i+1}|} \\
    z_1 \leq x + \epsilon  \Rightarrow \xi^+ &\in \mathbb{R}^{|x|} \\
    -z_1 \leq -x + \epsilon  \Rightarrow \xi^- &\in \mathbb{R}^{|x|} \\
    -z_{i,j} \leq 0  \Rightarrow \mu_{i,j} &\in \mathbb{R} \\
    \hat{z}_{i,j} - z_{i,j} \leq 0  \Rightarrow \tau_{i,j} &\in \mathbb{R} \\
    -u_{i,j} \hat{z}_{i,j} + (u_{i,j}-\ell_{i,j}) z_{i,j} \leq -u_{i,j} \ell_{i,j}  \Rightarrow \lambda_{i,j} &\in \mathbb{R}
  \end{split}
\end{equation}
where we note that can easily eliminate the dual variables corresponding to the
$z_{i,j} = 0$ and $z_{i,j} = \hat{z}_{i,j}$ from the optimization problem, so we
don't define explicit dual variables for these; we also note that $\mu_{i,j}$,
$\tau_{i,j}$, and $\lambda_{i,j}$ are only defined for $i,j$ such that $j \in
\mathcal{I}_i$, but we keep the notation as above for simplicity.  With these
definitions, the dual problem becomes

\begin{equation}
  \begin{aligned}
    &&\maximize &  \bigg(\big.- (x+\epsilon)^T \xi^+ + (x-\epsilon)^T \xi^-\\
    &&&-\sum_{i=1}^{k-1} \nu_{i+1}^T b_i + \sum_{i=2}^{k-1}\lambda_i^T (u_i \ell_i)\big.\bigg) \\
    &&\subjectto& \;\;\\
    && \nu_k &= -c \\
    && \nu_{i,j} &= 0, \; j \in \mathcal{I}^-_i \\
    && \nu_{i,j} &= (W_{i}^T\nu_{i+1})_j, \;\; j \in \mathcal{I}^+_i \\
    &&& \mkern-134mu \left . \begin{aligned}
         \bigg(\big.(u_{i,j} - \ell_{i,j}) \lambda_{i,j}&  \\
         - \mu_{i,j} - \tau_{i,j}\big.\bigg) & = (W_{i}^T\nu_{i+1})_j \\
         \nu_{i,j} &= u_{i,j} \lambda_{i,j} - \mu_i
      \end{aligned} \right \} \; \begin{aligned}
      							&i=2,\ldots,k-1\\
      							&j \in \mathcal{I}_i 
      							\end{aligned}\\
    && W_1^T\nu_2 &= \xi^+ - \xi^{-} \\
    && \lambda, \tau, \mu, \xi^{+}, \xi^{-} &\geq 0
  \end{aligned}
\end{equation}

The key insight we highlight here is that \emph{the dual problem can also be
written in the form of a deep network}, which provides a trivial way to find
feasible solutions to the dual problem, which can then be optimized over.
Specifically, consider the constraints
\begin{equation}
  \begin{split}
    (u_{i,j} - \ell_{i,j}) \lambda_{i,j} - \mu_{i,j} - \tau_{i,j} & = (W_{i}^T\nu_{i+1})_j \\
    \nu_{i,j} & = u_{i,j} \lambda_{i,j} - \mu_i.
  \end{split}
\end{equation}
Note that the dual variable $\lambda$ corresponds to the upper bounds in the
convex ReLU relaxation, while $\mu$ and $\tau$ correspond to the lower bounds $z
\geq0$ and $z \geq \hat{z}$ respectively; by the complementarity property, we
know that at the optimal solution, these variables will be zero if the ReLU
constraint is non-tight, or non-zero if the ReLU constraint is tight.  Because
we cannot have the upper and lower bounds be simultaneously tight (this would
imply that the ReLU input $\hat{z}$ would exceed its upper or lower bound
otherwise), we know that either $\lambda$ or $\mu + \tau$ must be zero.  This
means that at the optimal solution to the dual problem
\begin{equation}
  \begin{split}
  (u_{i,j} - \ell_{i,j}) \lambda_{i,j} & = [(W_{i}^T\nu_{i+1})_j]_+ \\
  \tau_{i,j} + \mu_{i,j} & = [(W_{i}^T\nu_{i+1})_j]_-
  \end{split}
\end{equation}
i.e., the dual variables capture the positive and negative portions of
$(W_{i}^T\nu_{i+1})_j$ respectively.  Combining this with the constraint that 
\begin{equation}
  \nu_{i,j} = u_{i,j} \lambda_{i,j} - \mu_i
\end{equation}
means that
\begin{equation}
  \nu_{i,j} = \frac{u_{i,j}}{u_{i,j} - \ell_{i,j}} [(W_{i}^T\nu_{i+1})_j]_+ -
  \alpha [(W_{i}^T\nu_{i+1})_j]_-
\end{equation}
for  $j \in \mathcal{I}_i$ and for some $0 \leq \alpha \leq 1$ (this accounts for the fact that we can either
put the ``weight'' of $[(W_{i}^T\nu_{i+1})_j]_-$ into $\mu$ or $\tau$, which
will or will not be passed to the next $\nu_i$).  This is exactly a type of
leaky ReLU operation, with a slope in the positive portion of $u_{i,j}/(u_{i,j}
- \ell_{i,j})$ (a term between 0 and 1), and a negative slope anywhere between
0 and 1.  Similarly, and more simply, note that $\xi^+$ and $\xi^{-}$ 
denote the positive and negative portions of $W_1^T \nu_2$, so we can replace
these terms with an absolute value in the objective.  Finally, we note
that although it is possible to have $\mu_{i,j} > 0$ and $\tau_{i,j} > 0$
simultaneously, this corresponds to an activation that is identically zero
pre-ReLU (both constraints being tight), and so is expected to be relatively
rare.   Putting this all
together, and using $\hat{\nu}$ to denote ``pre-activation'' variables in the
dual network, we can write the dual problem in terms of the network
\begin{equation}
  \begin{split}
    \nu_k & = -c \\
    \hat{\nu}_i & = W_i^T \nu_{i+1}, i=k-1,\ldots,1 \\
    \nu_{i,j} & = \left \{
      \begin{array}{ll}
        0 & j \in \mathcal{I}^-_i \\
        \hat{\nu}_{i,j} & j \in \mathcal{I}^+_i \\
        \frac{u_{i,j}}{u_{i,j} - \ell_{i,j}} [\hat{\nu}_{i,j}]_+ - \alpha_{i,j} [\hat{\nu}_{i,j}]_- & j \in \mathcal{I}_i, \\
      \end{array} \right .\\
      &\kern 1.5in \for i= k-1, \dots, 2
  \end{split}
\end{equation}
which we will abbreviate as $\nu = g_{\theta}(c,\alpha)$ to emphasize the fact
that $-c$ acts as the ``input'' to the network and $\alpha$ are per-layer inputs
we can also specify (for only those activations in $\mathcal{I}_i$), where
$\nu$ in this case is shorthand for all the $\nu_i$ and $\hat{\nu}_i$
activations.

The final objective we are seeking to optimize can also be written
\begin{equation}
  \begin{split}
    J_\epsilon(x, \nu)  =& -\sum_{i=1}^{k-1} \nu_{i+1}^T b_i - (x+\epsilon)^T[\hat{\nu}_1]_+ + (x-\epsilon)^T[\hat{\nu}_1]_- \\
    & + \sum_{i=2}^{k-1}\sum_{j \in \mathcal{I}_i} \frac{u_{i,j} \ell_{i,j}}{u_{i,j} - \ell_{i,j}}[\hat{\nu}_{i,j}]_+ \\
     =& -\sum_{i=1}^{k-1} \nu_{i+1}^T b_i - x^T\hat{\nu}_1 - \epsilon \|\hat{\nu}_1\|_1\\
     & + \sum_{i=2}^{k-1}\sum_{j \in \mathcal{I}_i} \ell_{i,j} [\nu_{i,j}]_+ \\
  \end{split}
\end{equation}
\end{proof}

\subsection{Justification for Choice in $\alpha$}
\label{app:alpha}
While any choice of $\alpha$ results in a lower bound via the dual problem, 
the specific choice of $\alpha = \frac{u_{i,j}}{u_{i,j} - \ell_{i,j}}$ is
also motivated by an alternate derivation of the dual problem from the perspective 
of general conjugate functions. 
 We can represent the adversarial problem from \eqref{eq:adversarial_polytope}
 in the following, general formulation 
\begin{equation}
  \begin{aligned}
  \minimize \;\; & c^T \hat{z}_k + f_1(z_1) + \sum_{i=2}^{k-1}f_i(\hat z_i, z_i)\\
  \subjectto \;\; & \hat{z}_{i+1} = W_i z_i + b_i, \; i=1,\ldots,k-1 \\
\end{aligned}
\end{equation}
where $f_1$ represents some input condition and $f_i$ represents some non-linear connection between layers. For example, we can take $f_i(\hat z_i, z_i) = I(\max(\hat z_i, 0) = z_i)$ to get ReLU activations, and take $f_1$ to be the indicator function for an $\ell_\infty$ ball with 
radius $\epsilon$ to get the adversarial problem in an $\ell_\infty$ ball for a ReLU network. 

Forming the Lagrangian, we get
\begin{equation}
  \begin{aligned}
  \mathcal L(z,\nu, \xi)  &= c^T \hat{z}_k + \nu_k^T \hat{z}_k + f_1(z_1) - \nu_2^TW_1z_1 \\
  &+ \sum_{i=2}^{k-1}\left( f_i(\hat z_i, z_i) - \nu_{i+1}^TW_i z_i + \nu_{i}^T\hat{z}_{i} \right)\\
  &-\sum_{i=1}^{k-1} \nu_{i+1}^Tb_i
\end{aligned}
\end{equation}

\paragraph{Conjugate functions} We can re-express this using conjugate functions
defined as 
$$f^*(y) = \max_x y^Tx - f(x)$$
but specifically used as 
$$-f^*(y) = \min_x f(x)-y^Tx $$

Plugging this in, we can minimize over each $\hat{z}_i, z_i$ pair independently 
\begin{equation}
\begin{aligned}
&\min_{z_1} f_1(z_1) - \nu_2^TW_1z_1 = -f_1^*(W_1^T\nu_2)\\
&\min_{\hat z_i, z_i} f_i(\hat z_i, z_i) - \nu_{i+1}^TW_iz_i + \nu_i^T \hat z_i \\
&\quad\quad\quad= -f_i^*(-\nu_i, W_i^T\nu_{i+1}), \; i=2, \ldots, k-1\\
&\min_{\hat z_k} c^T\hat z_k + \nu_k^T \hat z_k = I(\nu_k = -c)
\end{aligned}
\end{equation}
Substituting the conjugate functions into the Lagrangian, and letting $\hat \nu_i = W_i^T \nu_{i+1}$, we 
get
\begin{equation}
\label{eq:conjugate}
  \begin{aligned}
  \maximize_{\nu} \;\; &  -f^*_1(\hat \nu_1) - \sum_{i=2}^{k-1}f^*_i(-\nu_i, \hat \nu_i) -\sum_{i=1}^{k-1} \nu_{i+1}^Tb_i\\
  \subjectto \;\; & \nu_k = -c\\
  & \hat \nu_i = W_i^T \nu_{i+1}, \; i = 1, \dots, k-1
\end{aligned}
\end{equation}
This is almost the form of the dual network. The last step is to plug in the indicator function for the 
outer bound of the ReLU activation (we denote the ReLU polytope) for $f_i$ and derive $f^*_i$. 
\paragraph{ReLU polytope}
Suppose we have a ReLU polytope 
\begin{equation}
\begin{aligned}
\mathcal S_i = \{(\hat z_i, z_i) : \hat z_{i,j} &\geq 0, \\
z_{i,j} &\geq \hat z_{i,j}, \\
-u_{i,j}\hat z_{i,j} + (u_{i,j} - \ell_{i,j})z_{i,j} &\leq -u_{i,j}\ell_{i,j}\}
\end{aligned}
\end{equation}
So $I_\mathcal{S}$ is the indicator for this set, and $I^*_\mathcal{S}$ is its conjugate. We will omit subscripts $(i,j)$ for brevity, but we can do this case by case elementwise. 
\begin{enumerate}
\item If $u \leq 0$ then $\mathcal S \subset \{(\hat z, z) : z = 0\}$. \\
Then, $I^*_{\mathcal S}(\hat y, y) \leq \max_{\hat z} \hat y\cdot\hat z = I(\hat y=0)$.
\item If $\ell \geq 0$ then $\mathcal S \subset \{(\hat z, z) : \hat z = z\}$. \\
Then, $I^*_{\mathcal S}(\hat y, y) \leq \max_{z} \hat y\cdot z + y\cdot z = (\hat y + y)z = I(\hat y + y = 0)$. 
\item Otherwise $\mathcal S = \{(\hat z, z) : \hat z \geq 0, z \geq \hat z, -u\hat z + (u - \ell)z = -u\ell\}$. The maximum must occur either on the line $-u\hat z + (u-\ell)z = -u\ell$ 
over the interval $[0, u]$, or at the point $(\hat{z}, z) = (0,0)$ (so the maximum must have
value at least 0). We proceed to examine this last case. 
\end{enumerate}
Let $\mathcal{S}$ be the set of the third case. Then: 
\begin{equation}
\begin{aligned}
&I^*_{\mathcal{S}}(\hat y, y) \\
&= \left[\max_{0 < \hat z < u} y \cdot \frac{u}{u-\ell}(\hat z - \ell) + \hat y \cdot \hat z \right]_+\\
&= \left[\max_{0 < \hat z < u} \left(\frac{u}{u-\ell}y + \hat y\right) \hat z - \frac{u\ell}{u-\ell}y\right]_+\\
&= \left[\max_{0 < \hat z < u} y \cdot \frac{u}{u-\ell}(\hat z - \ell) + \hat y \cdot \hat z = g(\hat y, y) \right]_+\\
&= \left\{ \begin{aligned}
&\left[-\frac{u\ell}{u-\ell}y\right]_+ && \text{ if } \frac{u}{u-\ell}y + \hat y \leq 0\\
&\left[\left(\frac{u}{u-\ell}y + \hat y\right)u-\frac{u\ell}{u-\ell}y\right]_+ && \text{ if } \frac{u}{u-\ell}y + \hat y > 0
\end{aligned} \right.
\end{aligned}
\end{equation}
Observe that the second case is always larger than first, so we get a tighter
upper bound when $\frac{u}{u-\ell}y + \hat{y} \leq 0$. If we plug in $\hat{y} = -\nu$ and $y = \hat{\nu}$, this condition is equivalent to 
$$\frac{u}{u-\ell}\hat \nu \leq \nu$$
Recall that in the LP form, the forward pass in this case was defined by 
$$\nu = \frac{u}{u-\ell}[\hat{\nu}]_+ + \alpha[\hat{\nu}]_-$$
Then, $\alpha = \frac{u}{u-l}$ can be interpreted as the \emph{largest choice of $\alpha$ which
does not increase the bound} (because if $\alpha$ was any larger, we would enter the 
second case and add an additional $\left(\frac{u}{u-\ell}\hat{\nu} - \nu\right)u$ term to the bound). 

We can verify that using $\alpha = \frac{u}{u-\ell}$ results in the 
same dual problem by first simplifying the above to
$$I_\mathbb{S}^*(\-\nu, \hat \nu) = -l[{\nu}]_+$$
Combining this with the earlier two cases and plugging into \eqref{eq:conjugate} using $f^*_i = I^*_\mathcal{S}$ results in 
\begin{equation}
  \begin{aligned}
  \maximize_{\nu} \;\; &  -x^T\hat \nu_1 - f_1^*(\hat{\nu}_1)  -\sum_{i=1}^{k-1} \nu_{i+1}^Tb_i \\
                       &+ \sum_{i=2}^{k-1}\left(
  \sum_{j \in \mathcal I}l_{i,j} [\nu_{i,j}]_+\right)\\
  \subjectto \;\; \\
  & \mkern-18mu\nu_k = -c\\
  & \mkern-18mu\hat \nu_i = W_i^T \nu_{i+1}, \; i = 1, \dots, k-1\\
  & \mkern-18mu\nu_{i,j} = 0, \; i = 2,\ldots, k-1, \; j\in \mathcal I^-_i\\
  & \mkern-18mu\nu_{i,j} = \hat \nu_{i,j}, \;, i=2,\ldots, k-1, \; j \in \mathcal I^+_i\\
  & \mkern-18mu\nu_{i,j} = \frac{u_{i,j}}{u_{i,j}-l_{i,j}}\hat \nu_{i,j}, i=2, \ldots, k-1, \; j \in \mathcal I_i
\end{aligned}
\end{equation}
where the dual network here matches the one from \eqref{eq:dual_network} exactly when $\alpha = \frac{u_{i,j}}{u_{i,j}-l_{i,j}}$. 

\subsection{Proof of Theorem \ref{thm:robust_opt}}
In this section, we prove Theorem \ref{thm:robust_opt}, reproduced below:
\label{app:upper}
\begin{theorem*}
Let $L$ be a monotonic loss function that satisfies Property \ref{prop:invariance}. 
For any data point $(x,y)$, and $\epsilon>0$, the worst case adversarial loss from \eqref{eq:robust-opt} can be upper bounded with 
$$\max_{\|\Delta\|_\infty \leq \epsilon} L(f_\theta(x + \Delta), y) \leq L(-J_{\epsilon}(x, g_\theta(\mathbf{e}_{y} 1^T - I)), y)$$
where $J_{\epsilon}$ is as defined in \eqref{eq:J} for a given $x$ and $\epsilon$, and $g_\theta$ is as defined in \eqref{eq:dual_network} for the given model parameters $\theta$. 
\end{theorem*}
\begin{proof}
First, we rewrite the problem using the adversarial polytope $\mathcal{Z}_\epsilon(x)$. 
$$\max_{\|\Delta\|_\infty \leq \epsilon} L(f_\theta(x + \Delta), y) = \max_{\hat{z}_k\in \mathcal{Z}_\epsilon(x)} L(\hat{z}_k, y)$$
Since $L(x,y) \leq L(x-a1,y)$ for all $a$, we have
\begin{equation}
\begin{aligned}
\max_{\hat{z}_k\in {\mathcal{Z}}_\epsilon(x)} L(\hat{z}_k, y) &\leq \max_{\hat{z}_k\in {\mathcal{Z}}_\epsilon(x)} L(\hat{z}_k - (\hat{z}_k)_y1, y)\\
&= \max_{\hat{z}_k\in {\mathcal{Z}}_\epsilon(x)} L((I - \mathbf{e}_y1^T)\hat{z}_k, y)\\
&= \max_{\hat{z}_k\in {\mathcal{Z}}_\epsilon(x)} L(C\hat{z}_k, y)
\end{aligned}
\end{equation}
where $C = (I - \mathbf{e}_y1^T)$. 
Since $L$ is a monotone loss function, we can upper bound this further by using the element-wise maximum over $[C\hat{z}_k]_i$ for $i\neq y$, and elementwise-minimum for $i=y$ (note, however, 
that for $i=y$, $[C\hat{z}_k]_i=0$). Specifically, we bound it as
$$\max_{\hat{z}_k\in {\mathcal{Z}}_\epsilon(x)} L(C\hat{z}_k, y) \leq L(h(\hat z_k))$$
where, if $C_i$ is the $i$th row of $C$, $h(z_k)$ is defined element-wise as 
$$h(z_k)_i = \max_{\hat{z}_k\in {\mathcal{Z}}_\epsilon(x)}C_i\hat{z}_k$$
This is exactly the adversarial problem from \eqref{eq:adversarial_polytope} (in its 
maximization form instead of a minimization). Recall that $J$ from \eqref{eq:J} is a lower bound on \eqref{eq:adversarial_polytope} (using $c=-C_i$). 
\begin{equation}
\label{eq:lower_bound}
J_{\epsilon}(x, g_\theta(-C_i)) \leq \min_{\hat{z}_k \in \mathcal{Z}_\epsilon(x)} -C_i^T \hat{z_k}
\end{equation}
Multiplying both sides by $-1$ gives us the following upper bound
$$-J_{\epsilon}(x, g_\theta(-C_i)) \geq \max_{\hat{z}_k \in \mathcal{Z}_\epsilon(x)} C_i^T \hat{z_k}$$
Applying this upper bound to $h(z_k)_i$, we conclude
$$h(z_k)_i \leq -J_{\epsilon}(x, g_\theta(-C_i))$$
Applying this to all elements of $h$ gives the final upper bound on the adversarial loss. 
$$\max_{\|\Delta\|_\infty \leq \epsilon} L(f_\theta(x + \Delta), y) \leq L(-J_{\epsilon}(x, g_\theta(\mathbf{e}_{y} 1^T - I)), y)$$
\end{proof}

\subsection{Proof of Corollary \ref{thm:guarantee}}
In this section, we prove Corollary \ref{thm:guarantee}, reproduced below:
\label{app:guarantee}
\begin{theorem*}
For a data point $x$ and $\epsilon>0$, if 
\begin{equation}
\label{eq:certificate}
\min_{y \neq f(x)} [J_{\epsilon}(x, g_\theta(\mathbf{e}_{f(x)} 1^T - I, \alpha))]_{y} \geq 0
\end{equation} 
then
the model is guaranteed to be robust around this data point. Specifically,
there does not exist an adversarial example $\tilde x$ 
such that $|\tilde x - x|_\infty \leq \epsilon$ and $f_\theta(\tilde x) \neq f_\theta(x)$. 
\end{theorem*}
\begin{proof}
Recall that $J$ from \eqref{eq:J} is a lower bound on \eqref{eq:adversarial_polytope}. 
Combining this fact with the certificate in \eqref{eq:certificate}, we get that
for all $y \neq f(x)$, 
$$\min_{\hat{z}_k \in \mathcal{Z}_\epsilon(x)}  (\hat{z}_k)_{f(x)} - (\hat{z}_k)_{y} \geq 0$$
Crucially, this means that for every point in the adversarial polytope and 
for any alternative label $y$, $(\hat{z}_k)_{f(x)} \geq (\hat{z}_k)_{y}$, 
so the classifier cannot change its output within the adversarial polytope and is robust around $x$. 
\end{proof}

\section{Experimental Details}
\label{app:experiments}

\subsection{2D Example}
\subparagraph{Problem Generation}
We incrementally randomly sample 12 points within the $[0,1]$
$xy$-plane, at each point waiting until we find a sample that is at least $0.16$
away from other points via $\ell_\infty$ distance, and assign each point a
random label.  We then attempt to learn a robust classifier that will correctly
classify all points with an $\ell_\infty$ ball of $\epsilon = 0.08$. 
\subparagraph{Parameters}
 We use the Adam
optimizer \citep{kingma2015adam} (over the entire batch of samples) with a learning rate of
0.001.  

\begin{figure*}[t]
  \begin{center}
    \includegraphics[width=2in]{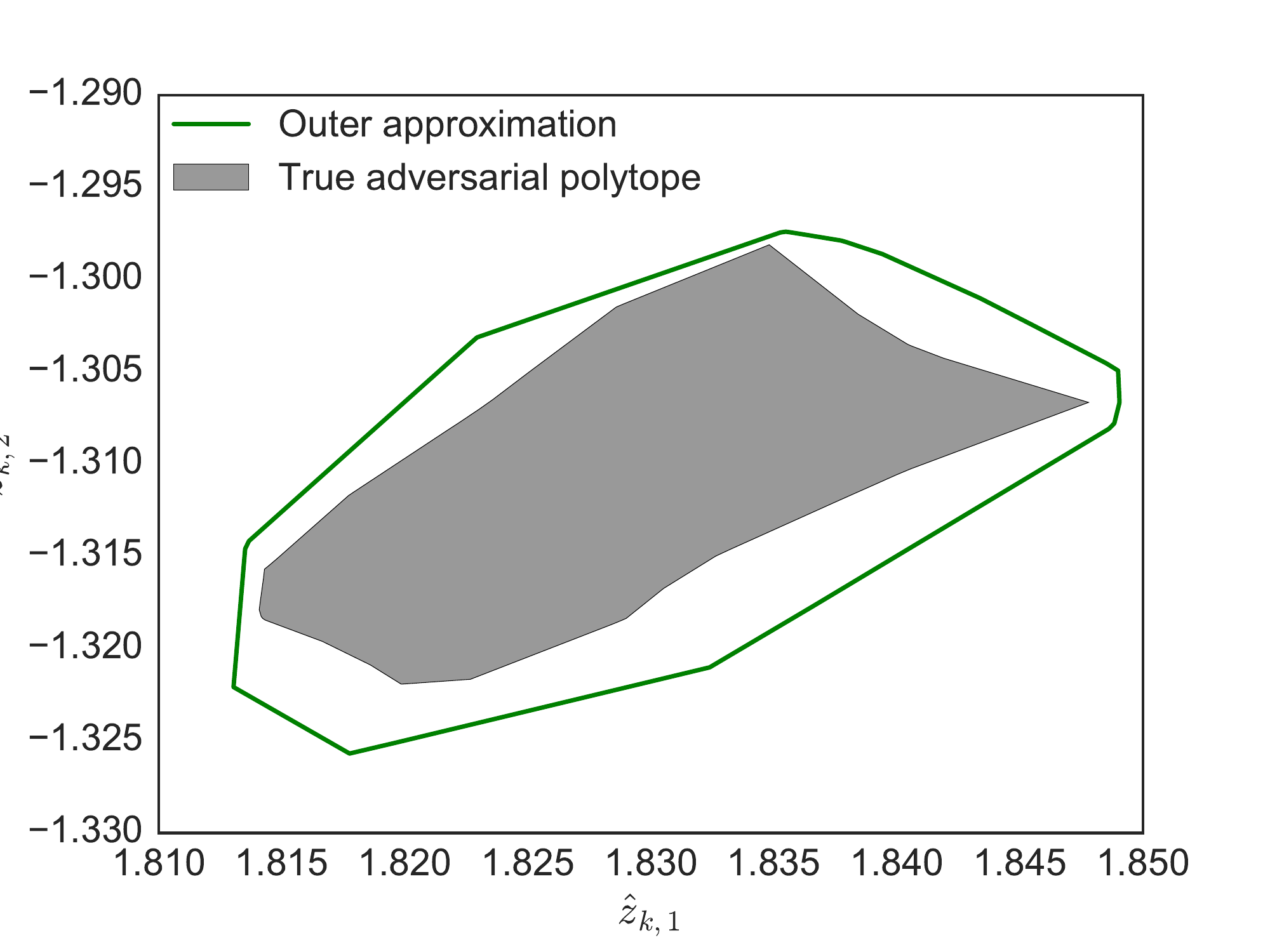}
    \includegraphics[width=2in]{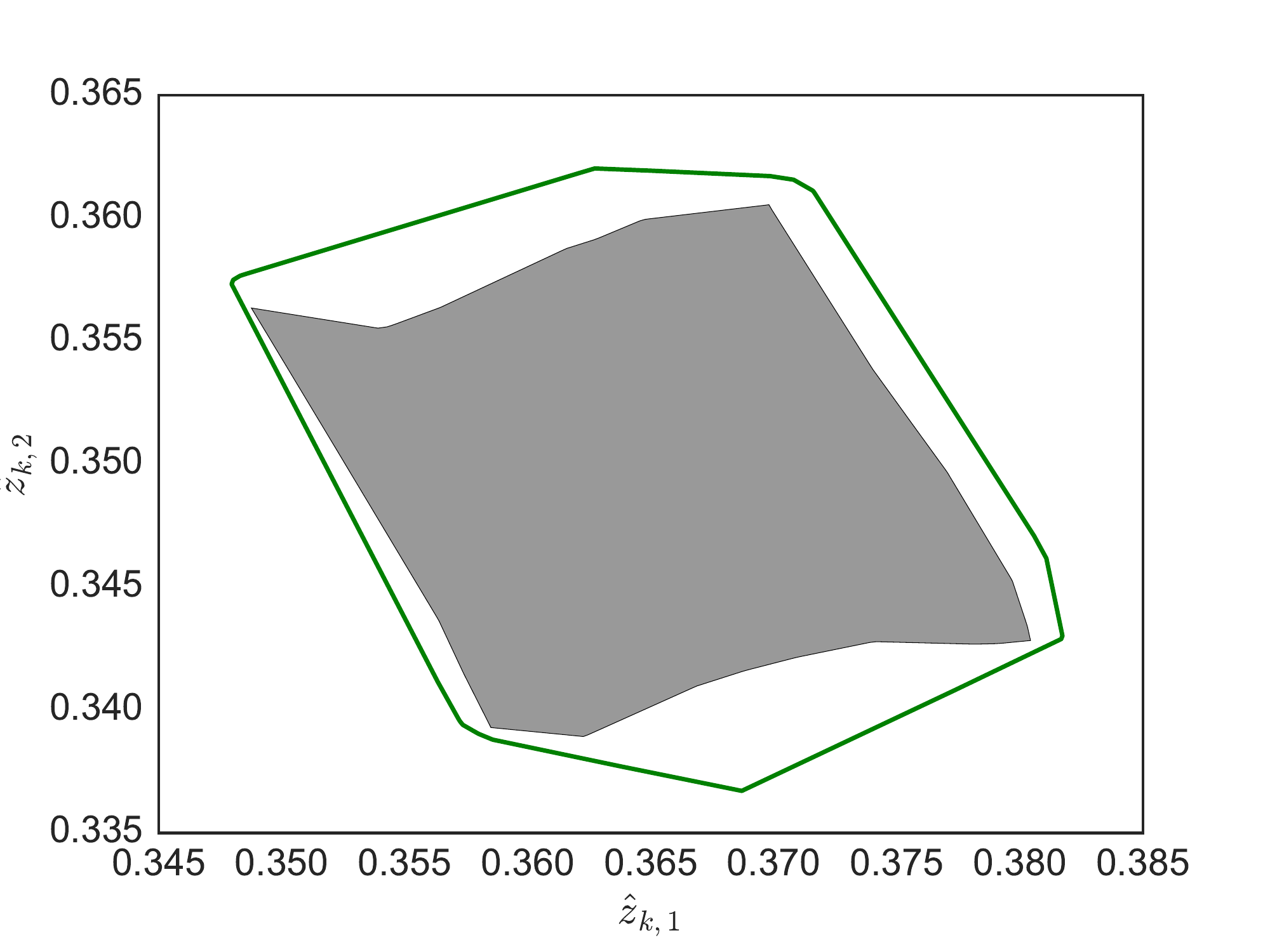}
    \includegraphics[width=2in]{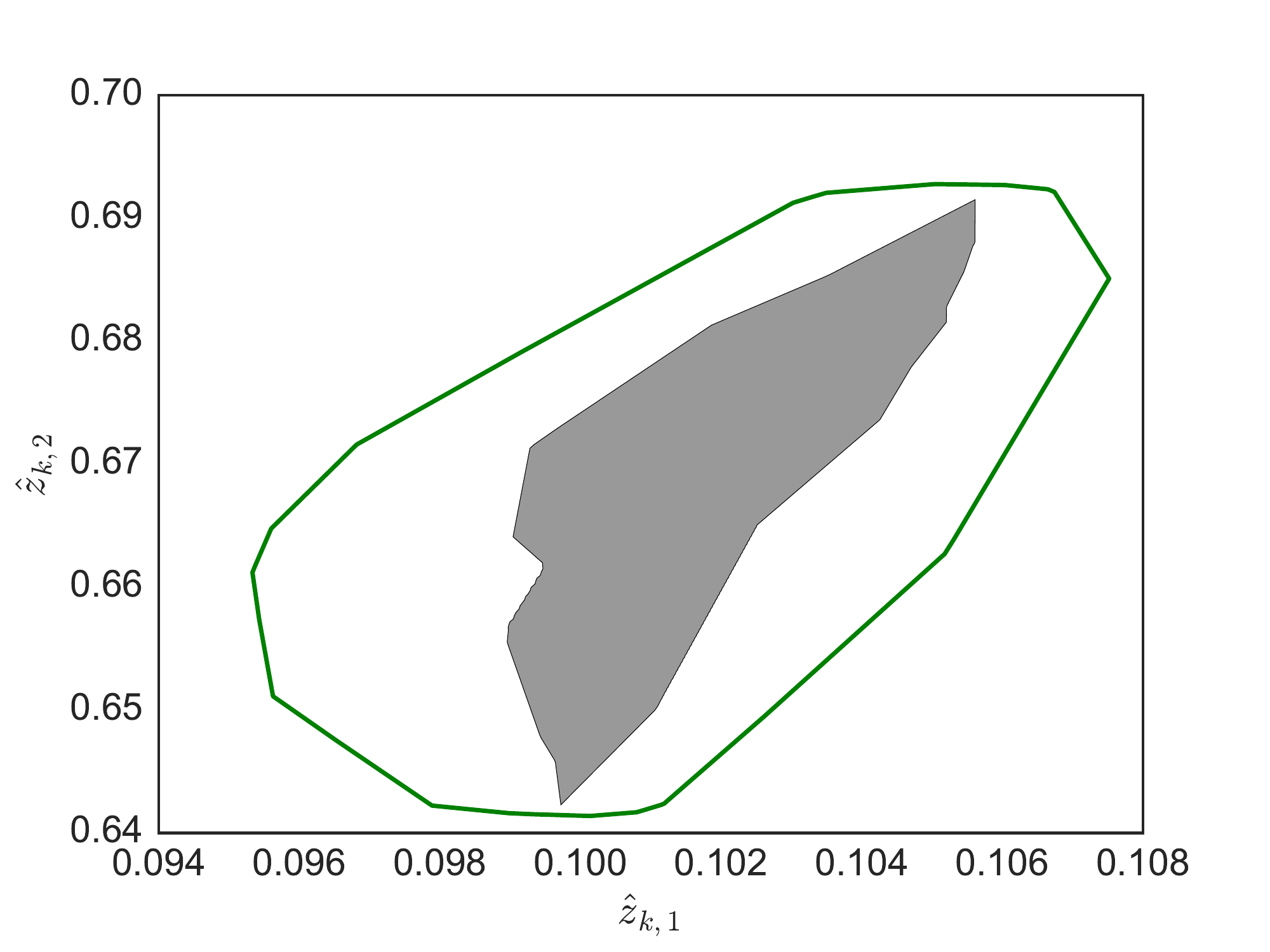} \\
    \includegraphics[width=2in]{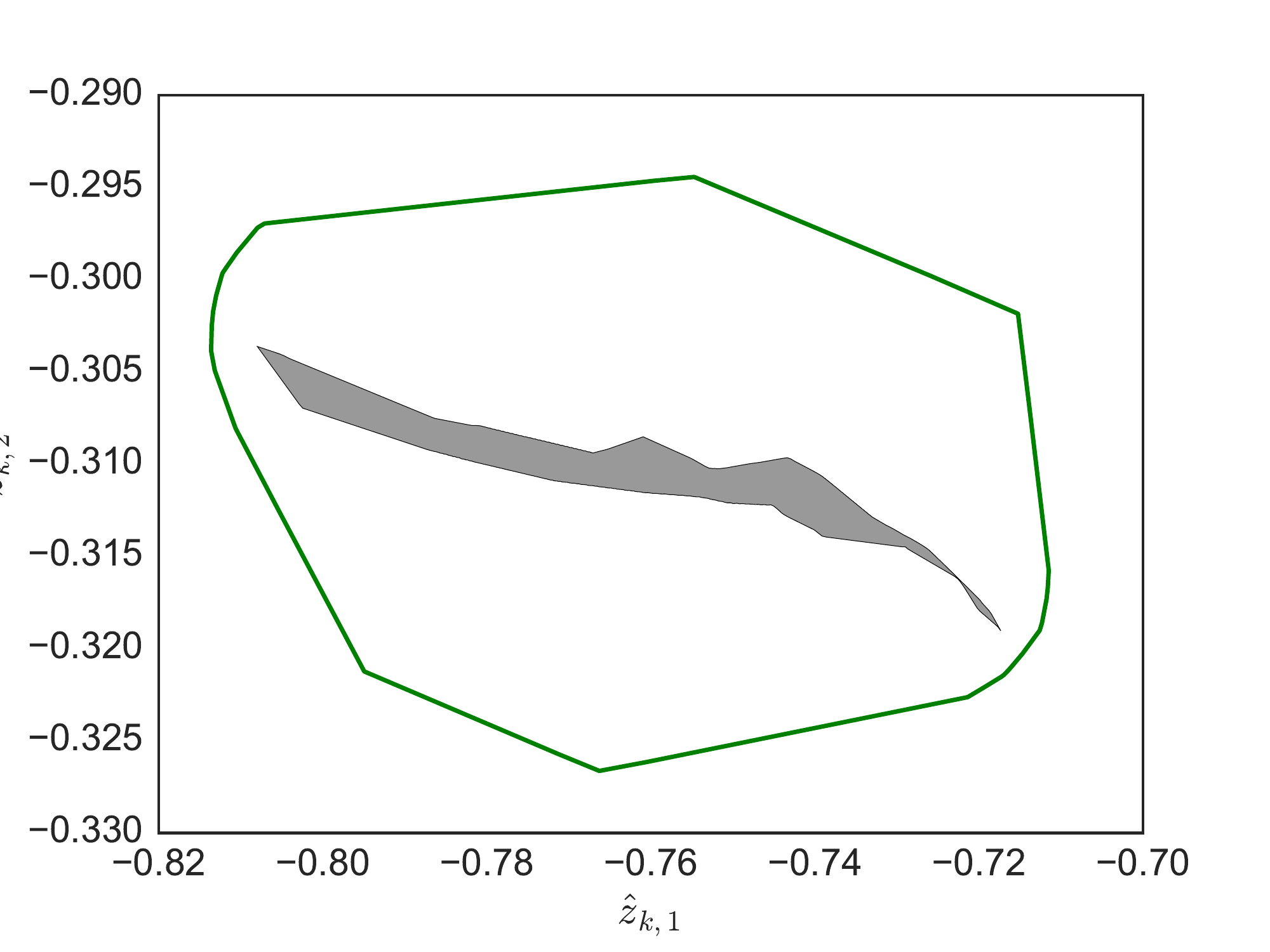}
    \includegraphics[width=2in]{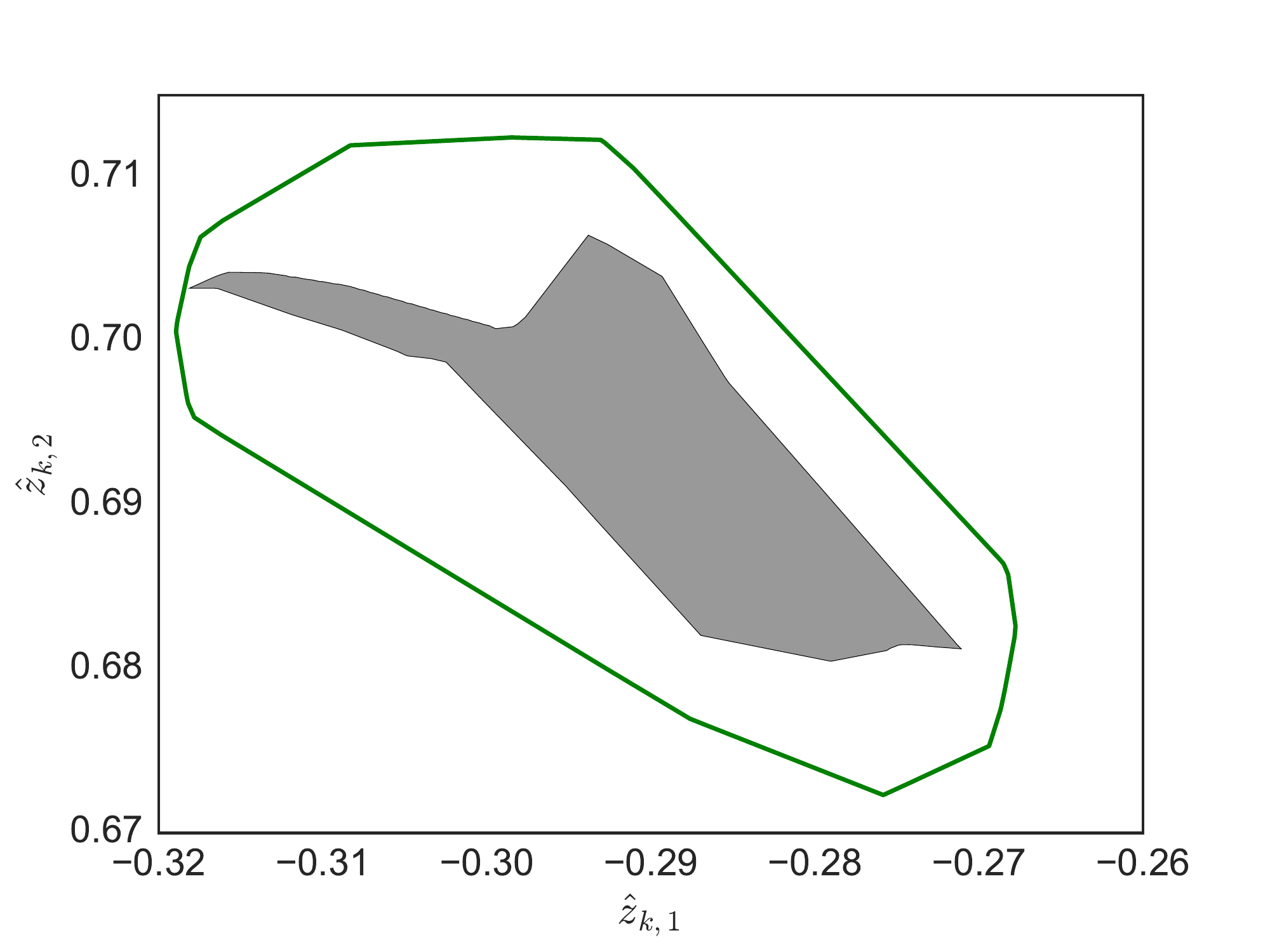}
    \includegraphics[width=2in]{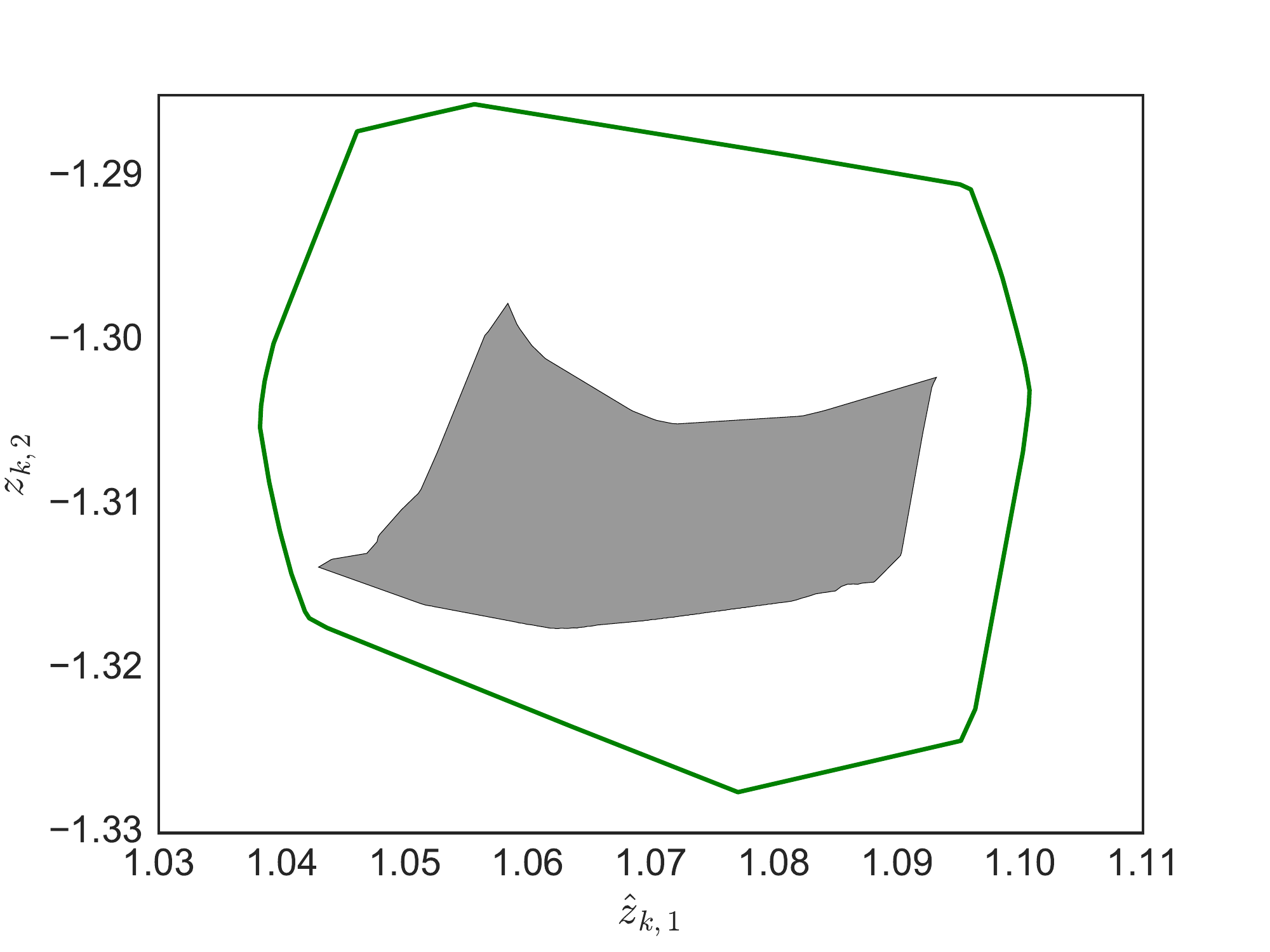} \\
    \includegraphics[width=2in]{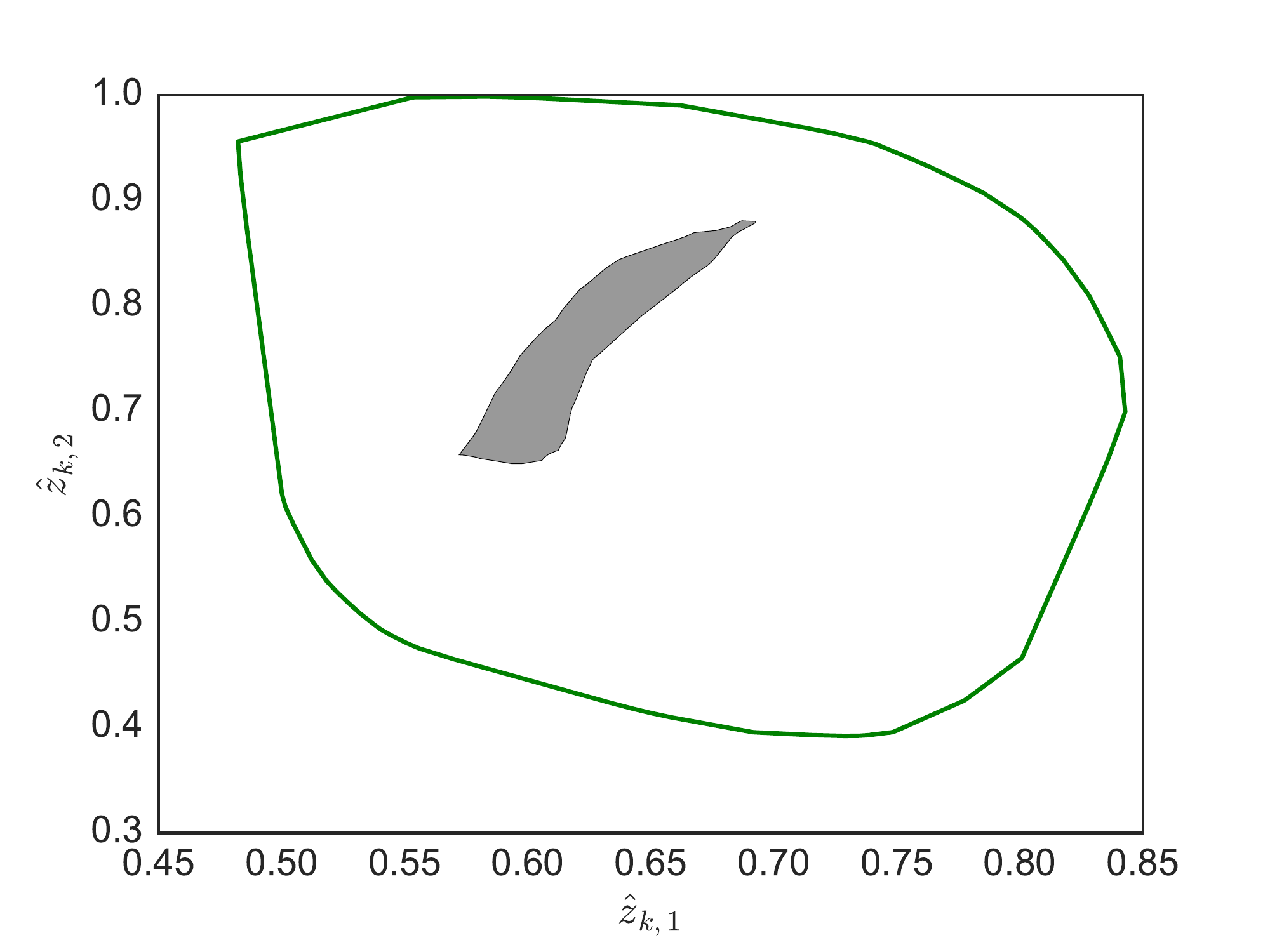}
    \includegraphics[width=2in]{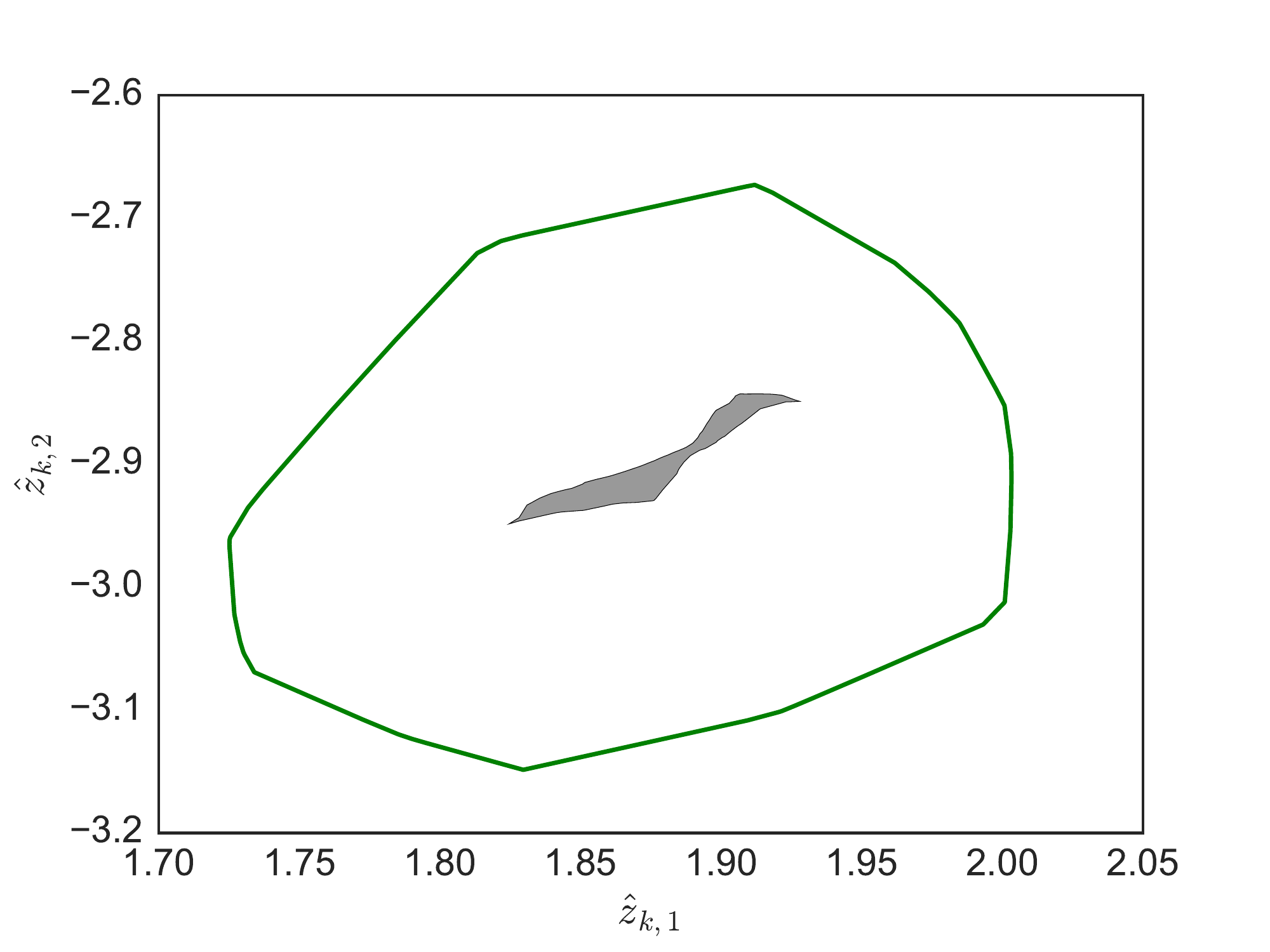}
    \includegraphics[width=2in]{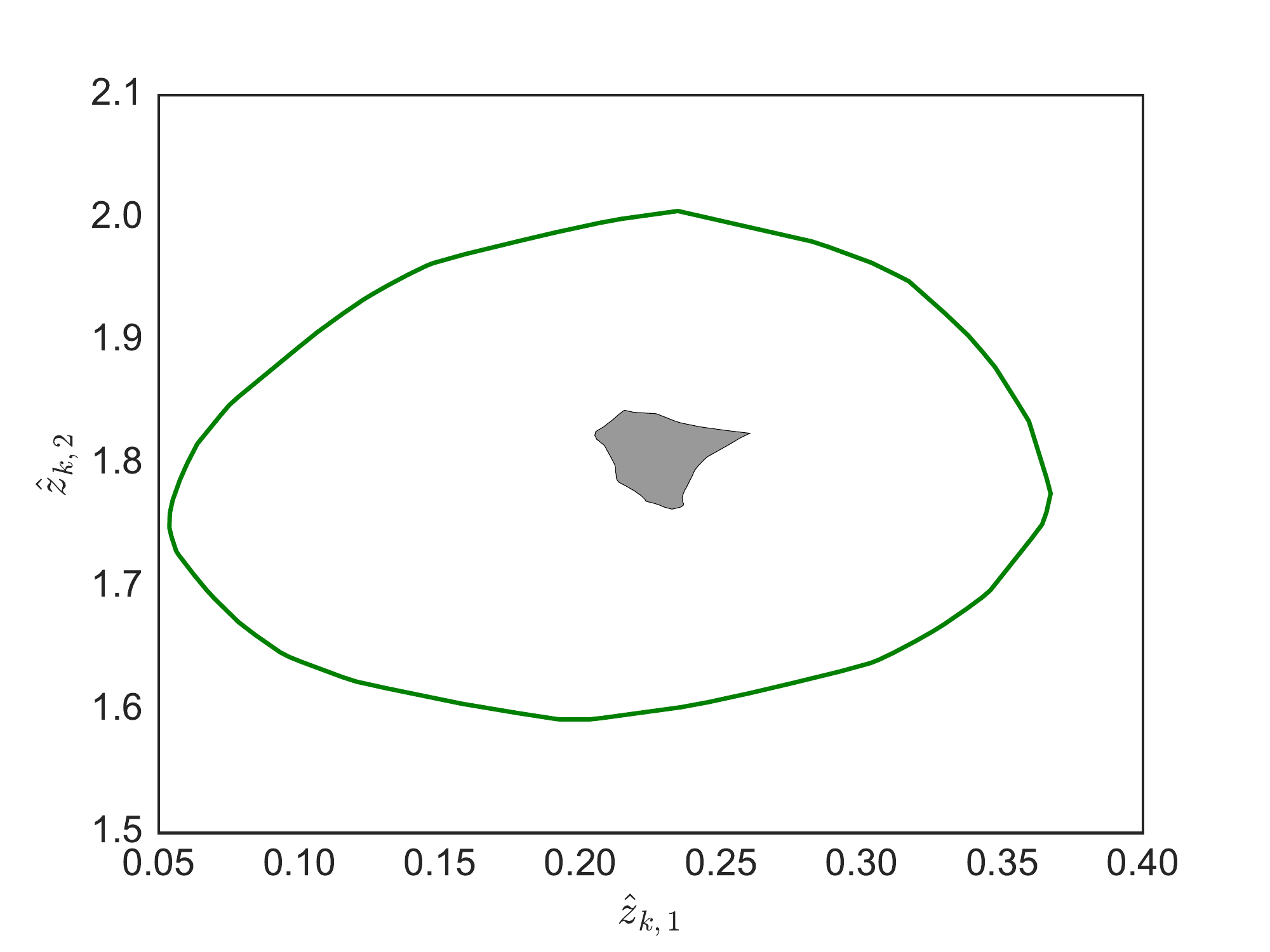}
  \end{center}
  \caption{Illustrations of the true adversarial polytope (gray) and our convex
    outer approximation (green) for a random 2-100-100-100-100-2 network with
    $\mathcal{N}(0,1/\sqrt{n})$ weight initialization.  Polytopes are shown for
    $\epsilon=0.05$ (top row), $\epsilon=0.1$ (middle row), and $\epsilon=0.25$
    (bottom row).}
  \label{fig-polytopes}
\end{figure*}    

\subparagraph{Visualizations of the Convex Outer Adversarial Polytope} 
\label{app:tight}
We
consider some simple cases of visualizing the outer approximation to the
adversarial polytope for random networks in Figure \ref{fig-polytopes}.  
Because the output space is two-dimensional we can easily 
visualize the polytopes in the output layer, and because the input space is two
dimensional, we can easily cover the entire input space densely to enumerate the
true adversarial polytope.  In this experiment, we initialized the weights of
the all layers to be normal $\mathcal{N}(0,1/\sqrt{n_{\mathrm{in}}})$ and biases
normal $\mathcal{N}(0,1)$ (due to scaling, the actual absolute value of weights
is not particularly important except as it relates to $\epsilon$).  Although
obviously not too much should be read into these experiments with random
networks, the main takeaways are that 1) for ``small'' $\epsilon$, the outer bound
is an extremely good approximation to the adversarial polytope; 2) as $\epsilon$
increases, the bound gets substantially weaker.  This is to be expected: for
small $\epsilon$, the number of elements in $\mathcal{I}$ will also be
relatively small, and thus additional terms that make the bound lose are
expected to be relatively small (in the extreme, when no activation can change,
the bound will be exact, and the adversarial polytope will be a convex set).
However, as $\epsilon$ gets larger, more activations enter the set
$\mathcal{I}$, and the available freedom in the convex relaxation of each ReLU
increases substantially, making the bound looser.   Naturally, the question
of interest is how tight this bound is for networks that are actually trained to
minimize the robust loss, which we will look at shortly.

\paragraph{Comparison to Naive Layerwise Bounds}  One additional point is worth
making in regards to the bounds we propose.  It would also be possible to
achieve a naive ``layerwise'' bound by iteratively determining absolute
allowable ranges for each activation in a network (via a simple norm bound), 
then for future layers, assuming each activation can vary
arbitrarily within this range.  This provides a simple iterative formula for
computing layer-by-layer absolute bounds on the coefficients, and similar
techniques have been used e.g. in Parseval Networks \citep{cisse2017parseval} to
produce more 
robust classifiers (albeit there considering $\ell_2$ perturbations instead of
$\ell_\infty$ perturbations, which likely are better suited for such an
approach).  Unfortunately, these naive bounds are extremely loose for
multi-layer networks (in the first hidden layer, they naturally match our bounds
exactly).  For instance, for the adversarial polytope shown in Figure
\ref{fig-polytopes} (top left), the actual adversarial polytope is contained
within the range
\begin{equation}
  \hat{z}_{k,1} \in [1.81, 1.85], \;\; \hat{z}_{k,2} \in [-1.33, -1.29]
\end{equation}
with the convex outer approximation mirroring it rather closely.  In
contrast, the layerwise bounds produce the bound:
\begin{equation}
  \hat{z}_{k,1} \in [-11.68, 13.47], \;\; \hat{z}_{k,2} \in [-16.36, 11.48].
\end{equation}
Such bounds are essentially vacuous in our case, which makes sense intuitively.
The naive bound has no way to exploit the ``tightness'' of activations that lie
entirely in the positive space, and effectively replaces the convex ReLU
approximation with a (larger) box covering the entire space.  Thus, such bounds
are not of particular use when considering robust classification.

\subparagraph{Outer Bound after Training}
\begin{figure}[t]
\vskip 0.2in
\begin{center}
\centerline{\includegraphics[width=\columnwidth]{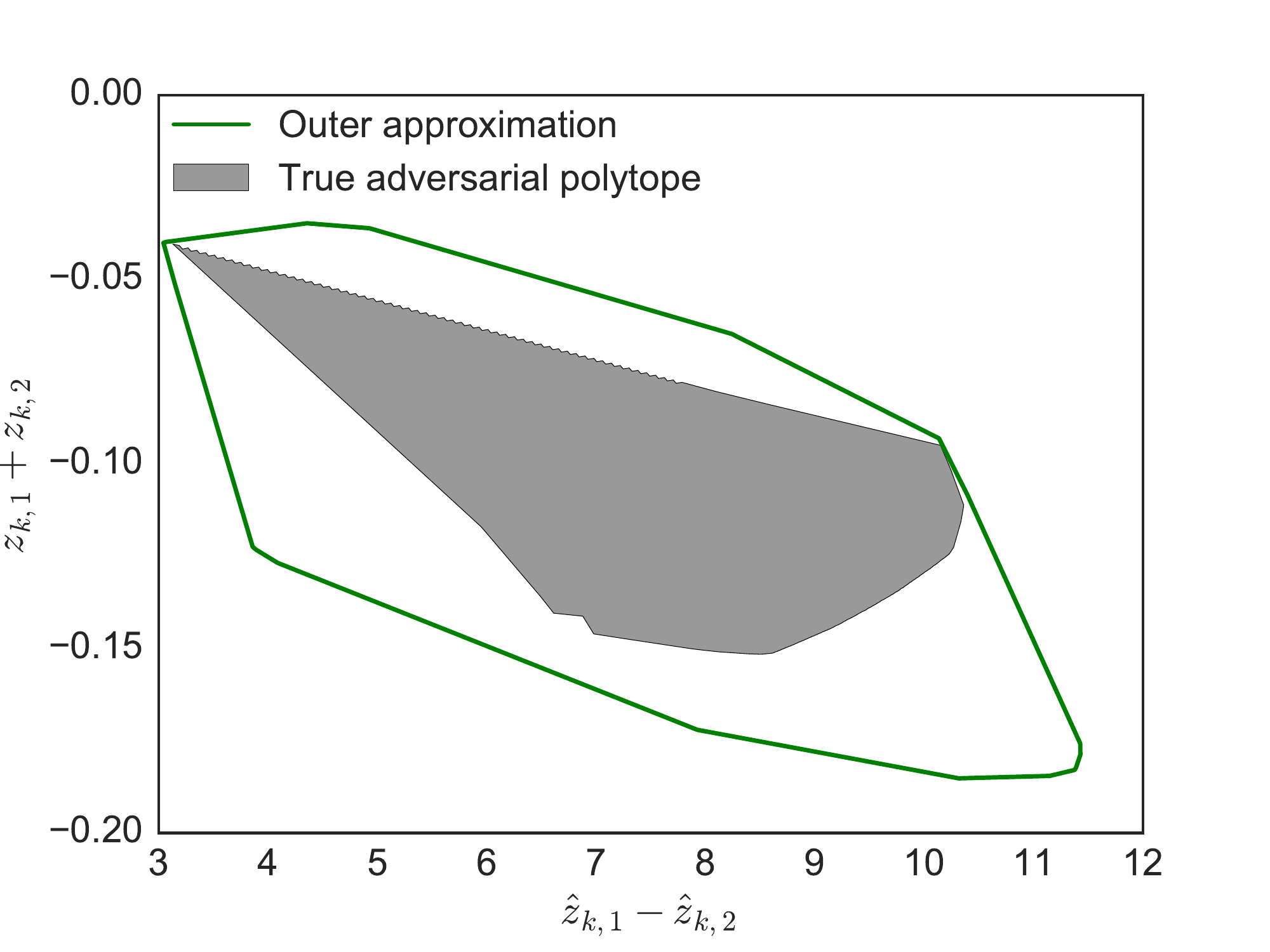}}
  \caption{Illustration of the
  actual adversarial polytope and the convex outer approximation for one of the
  training points after the robust optimization procedure.} 
  \label{fig-robust-polytopes}
\end{center}
\vskip -0.2in
\end{figure}
It is of some interest to see what the true adversarial
polytope for the examples in this data set looks like versus the convex
approximation, evaluated at the solution of the robust optimization problem.
Figure \ref{fig-robust-polytopes} shows one of these figures, highlighting the
fact that for the final network weights and choice of epsilon, the outer bound
is empirically quite tight in this case. In Appendix \ref{app:mnist} we calculate
exactly the gap between the primal problem and the dual bound on the MNIST 
convolutional model. 
In Appendix \ref{app:har}, we will see that when training
on the HAR dataset, even for larger $\epsilon$, the bound is empirically
tight.

\subsection{MNIST}
\label{app:mnist}
\subparagraph{Parameters} We use the Adam optimizer
\citep{kingma2015adam} with a learning rate of 0.001 (the default option) with
no additional hyperparameter selection. We use minibatches of size 50 and 
train for 100 epochs. 

\subparagraph{$\epsilon$ scheduling}
Depending on the random weight initialization of the network, the optimization 
process for training a robust MNIST classifier may get stuck and not converge. 
To improve convergence, it is helpful to start with a smaller value 
of $\epsilon$ and slowly increment it over epochs. For MNIST, all random seeds 
that we observed to not converge for $\epsilon=0.1$ were able to 
converge when started with
$\epsilon=0.05$ and taking uniform steps to $\epsilon=0.1$ in the first half of all epochs (so in this case, 50 epochs).

 \begin{figure}[t]
 \vskip 0.2in
 \begin{center}
 \centerline{\includegraphics[width=\columnwidth]{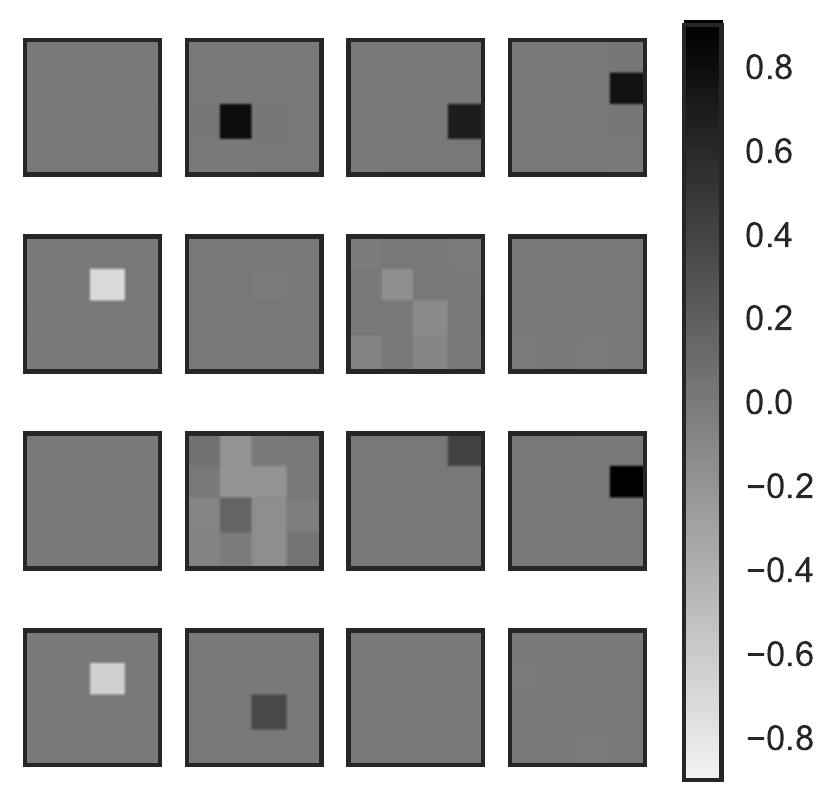}}
 \caption{Learned convolutional filters for MNIST of the first layer of a trained robust convolutional network, which are quite sparse due to the $\ell_1$ term in \eqref{eq:J}.}
 \label{fig:mnist_weights}
 \end{center}
 \vskip -0.2in
 \end{figure}

\begin{figure}[t]
\vskip 0.2in
\begin{center}
\centerline{\includegraphics[width=\columnwidth]{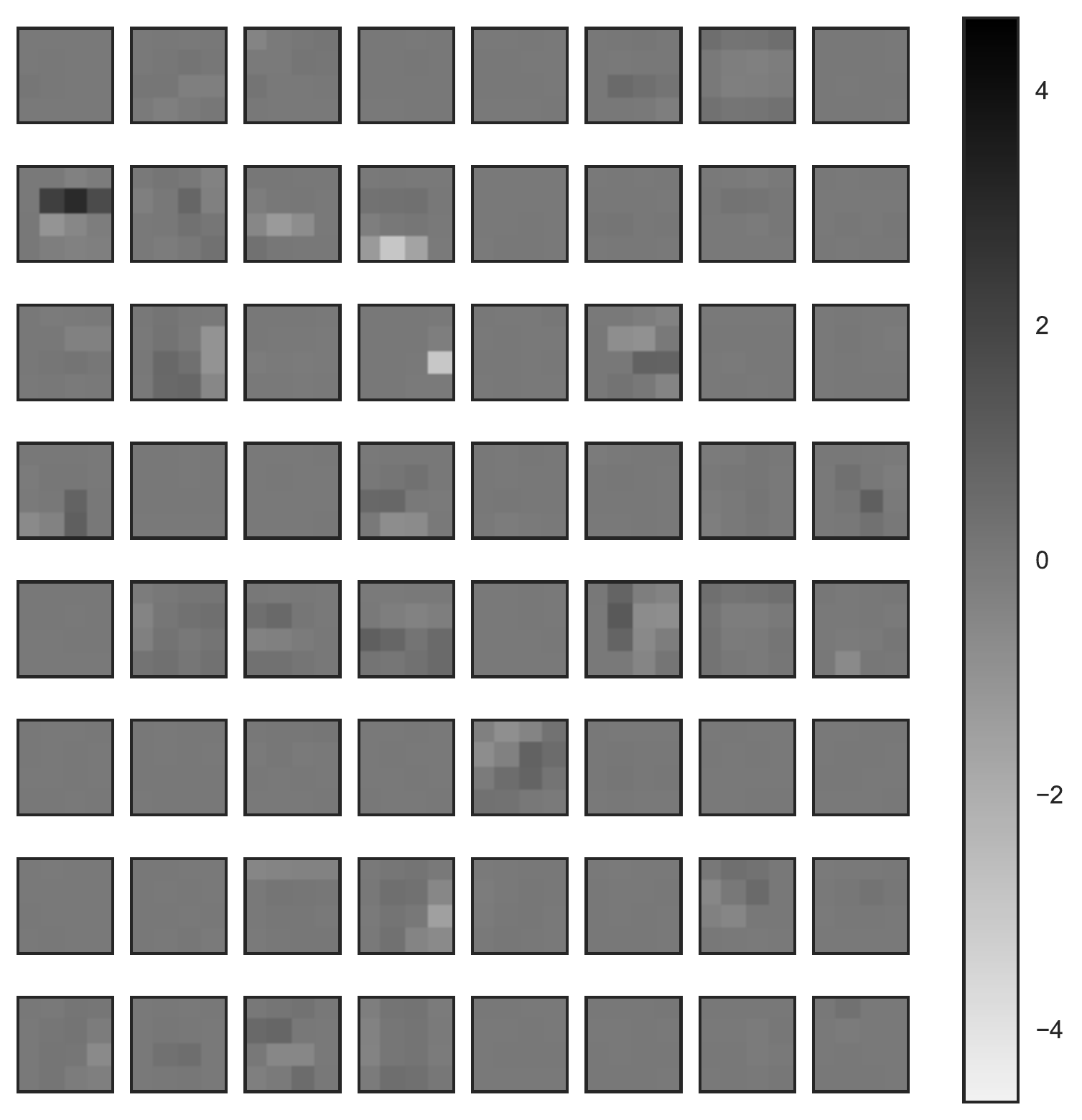}}
\caption{Learned convolutional filters for MNIST of the second layer of a trained robust convolutional network, which are quite sparse due to the $\ell_1$ term in \eqref{eq:J}.}
\label{fig:mnist_weights2}
\end{center}
\vskip -0.2in
\end{figure}
\subparagraph{MNIST convolutional filters}
\label{app:weights}
Random filters from the two convolutional layers of the MNIST classifier 
after robust training are plotted in Figure \ref{fig:mnist_weights2}. 
We see a similar story in both layers: they are highly sparse, 
and some filters have all zero weights. 

\begin{figure}[t]
\vskip 0.2in
\begin{center}
\centerline{\includegraphics[width=\columnwidth]{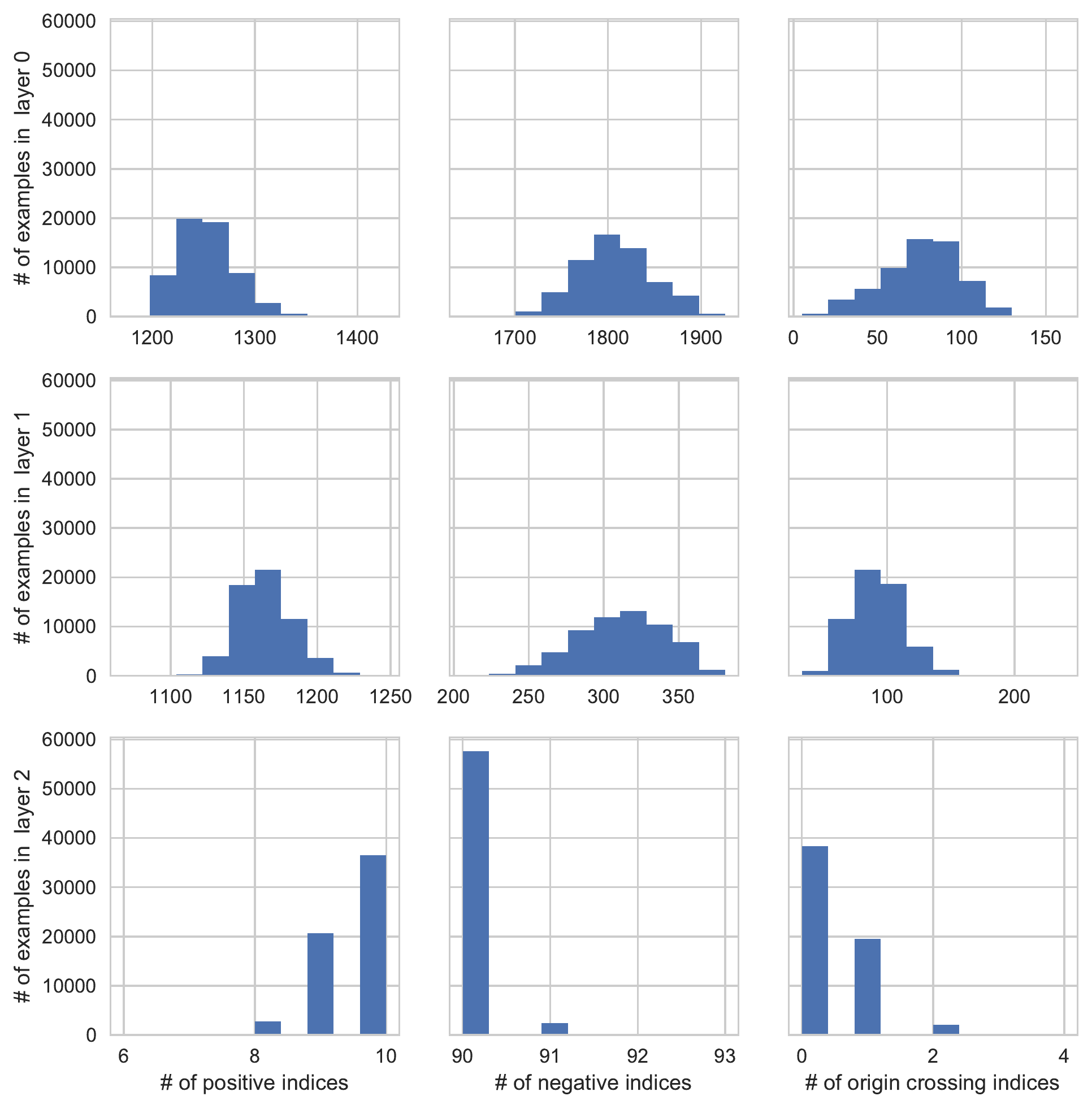}}
\caption{Histograms of the portion of each type of index set (as defined in \ref{eq:d_after_alpha} when passing training examples through the network.}
\label{fig:mnist_activations}
\end{center}
\vskip -0.2in
\end{figure}
\subparagraph{Activation index counts}
\label{app:activations}
We plot histograms to visualize the distributions of pre-activation bounds over examples 
in Figure \ref{fig:mnist_activations}. We see that in the first layer, examples have 
on average 
more than half of all their activations in the $\mathcal{I}^-_1$ set, with a relatively
small number of activations in the $\mathcal{I}_1$ set. The second layer has significantly more
values in the $\mathcal{I}^+_2$ set than in the $\mathcal{I}^-_2$ set, with a comparably small number of
activations in the $\mathcal{I}_{2}$ set. The third layer has extremely few activations in the $\mathcal{I}_{3}$ set, with 90\% all of the activations in the $\mathcal{I}^-_3$ set. Crucially, we see that in all three
layers, the number of activations in the $\mathcal I_i$ set is small, which benefits the
method in two ways: a) it makes the bound
tighter (since the bound is tight for activations through the $\mathcal I^+_i$ and $\mathcal I^-_i$ sets) 
and b) it makes the bound more computationally efficient to compute (since the
last term of \eqref{eq:J} is 
only summed over activations in the $\mathcal{I}_i$ set). 

\begin{figure*}[t]
\vskip 0.2in
\begin{center}
\centerline{\includegraphics[width=2\columnwidth]{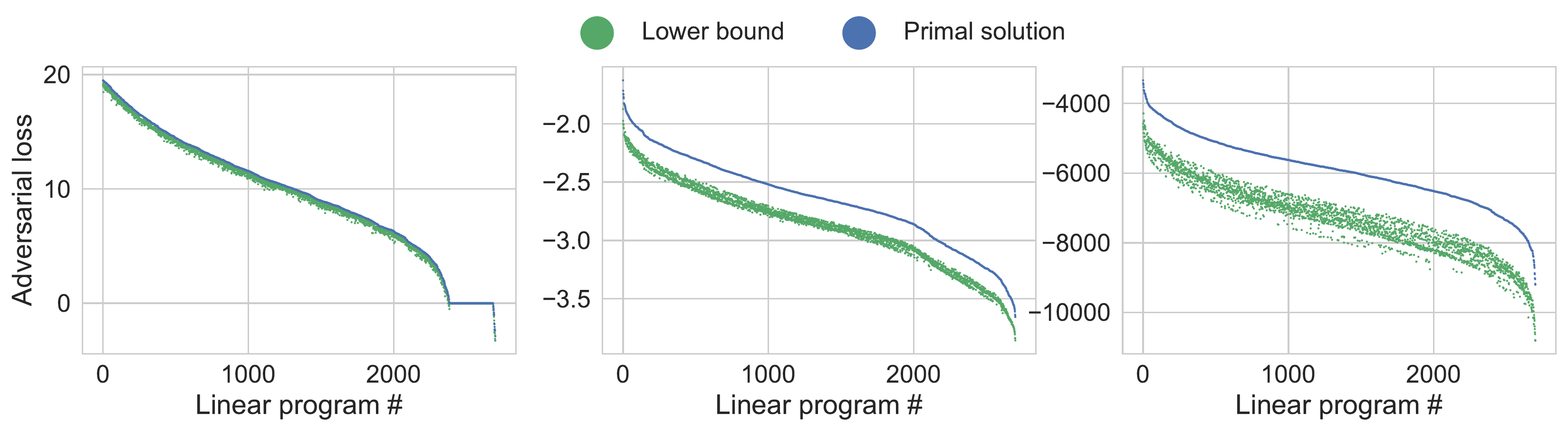}}
\caption{Plots of the exact solution of the primal linear program and the corresponding lower bound from the dual problem for a (left) robustly trained model, (middle) randomly intialized model, and (right) model with standard training.}
\label{fig:primal}
\end{center}
\vskip -0.2in
\end{figure*}

\subparagraph{Tightness of bound}
We empirically evaluate the tightness of the bound by exactly computing 
the primal LP and comparing it to the lower bound computed 
from the dual problem via our method. 
We find that the bounds, when computed on the robustly trained classifier, 
are extremely tight, especially when compared to bounds computed for
random networks and networks that have been trained under standard training, 
as can be seen in Figure \ref{fig:primal}. 

\subsection{Fashion-MNIST}
\subparagraph{Parameters}
We use exactly the same parameters as for MNIST: Adam optimizer with the default 
learning rate 0.001, minibatches of size 50, and trained for 100 epochs.

\begin{figure}[t]
\vskip 0.2in
\begin{center}
\centerline{\includegraphics[width=\columnwidth]{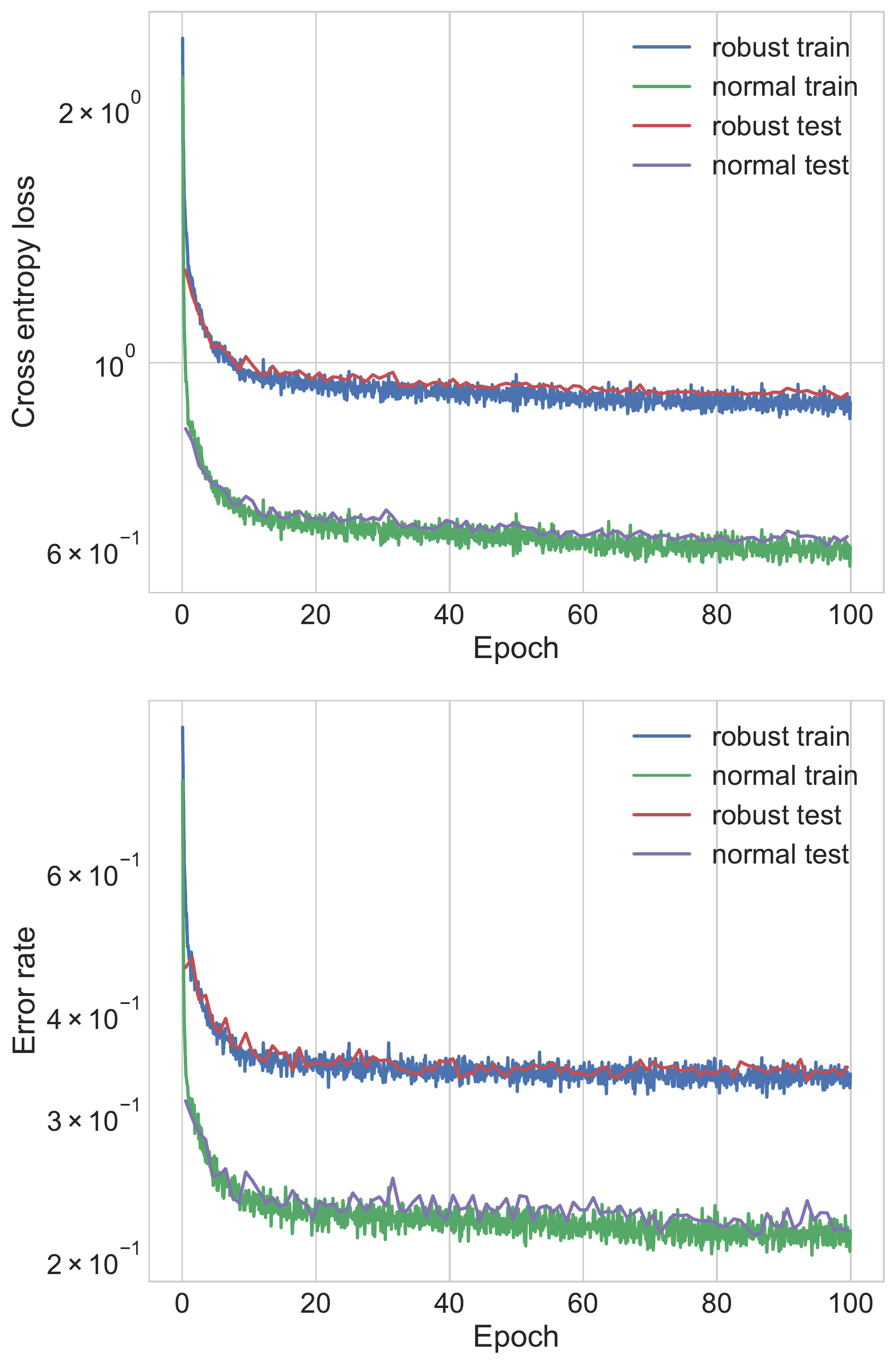}}
\caption{Loss (top) and error rate (bottom) when training a robust convolutional network on the Fashion-MNIST dataset. }
\label{fig:fashion}
\end{center}
\vskip -0.2in
\end{figure}

\subparagraph{Learning curves}
Figure \ref{fig:fashion} plots the error and loss curves (and their robust 
variants) of the model over epochs. We observe no overfitting, and suspect that the performance 
on this problem is limited by model capacity. 
\label{app:fashion}

\begin{figure}[t]
\vskip 0.2in
\begin{center}
\centerline{\includegraphics[width=\columnwidth]{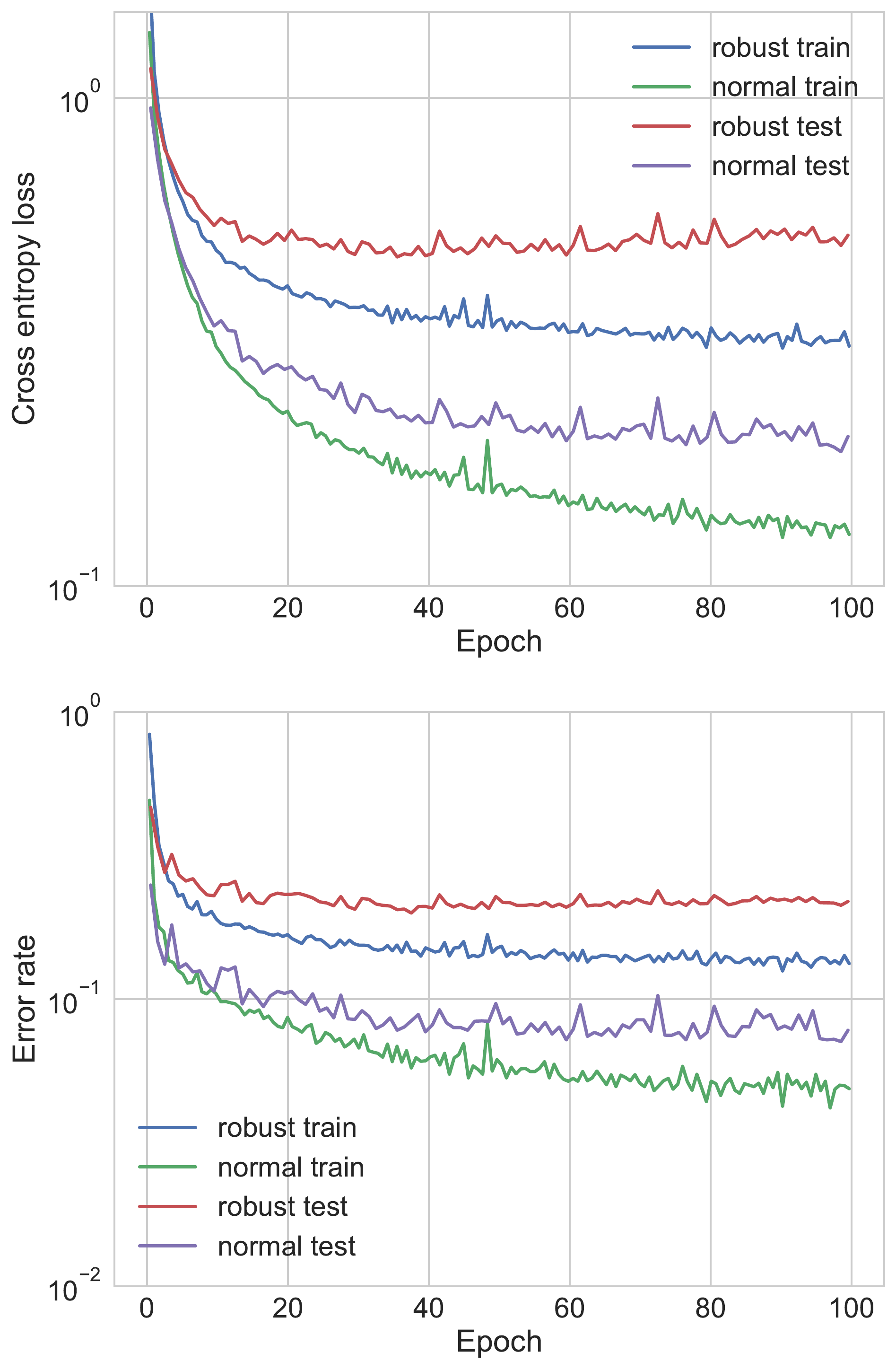}}
\caption{Loss (top) and error rate (bottom) when training a robust fully connected network on the 
HAR dataset with one hidden layer of 500 units. }
\label{fig:har}
\end{center}
\vskip -0.2in
\end{figure}
\subsection{HAR}
\label{app:har}
\subparagraph{Parameters}
We use the Adam optimizer with a learning rate 0.0001,
minibatches of size 50, and trained for 100 epochs. 

\subparagraph{Learning Curves}
Figure \ref{fig:har} plots the error and loss curves (and their robust variants) of the model over epochs.
 The bottleneck here is likely due to the simplicity of the problem
and the difficulty level implied by the value of $\epsilon$, as 
we observed that  
scaling to more more layers in this setting did not help. 

\begin{table}[t]
\caption{Tightness of the bound on a single layer neural network with 500 hidden units 
after training on the HAR dataset with
various values of $\epsilon$. We observe that regardless of how large $\epsilon$ is,
after training, the bound matches the error achievable by FGSM, implying 
that in this case the robust bound is tight.}
\label{tab:har-epsilon}
\vskip 0.15in
\begin{center}
\begin{small}
\begin{sc}
\begin{tabular}{cccccccc}
\toprule
$\epsilon$ & Test error & FGSM error & Robust bound\\
\midrule
0.05    & 9.20\% & 22.20\% & 22.80\% \\
0.1     & 15.74\% & 36.62\% & 37.09\% \\
0.25     & 47.66\% & 64.24\% & 64.47\% \\
0.5     & 47.08\% & 67.32\% & 67.86\% \\
1     & 81.80\% & 81.80\% & 81.80\% \\
\bottomrule
\end{tabular}
\end{sc}
\end{small}
\end{center}
\vskip -0.1in
\end{table}

\subparagraph{Tightness of bound with increasing $\epsilon$}
Earlier, we observed that on random networks, the bound gets progressively looser 
with increasing $\epsilon$ in Figure \ref{fig-polytopes}. 
In contrast, we find that even if we vary
the value of $\epsilon$, 
after robust training on the HAR dataset with a single hidden layer, 
the bound \emph{still} stays quite tight, 
as seen in Table \ref{tab:har-epsilon}. As expected, training a robust
model with larger $\epsilon$ results in a less accurate model since 
the adversarial problem is more difficult (and
potentially impossible to solve for some data points), 
however the key point is that
the robust bounds are extremely close to the achievable error rate by 
FGSM, implying that in this case, the bound is tight. 

\subsection{SVHN}
\label{app:svhn}
\begin{figure}[t]
\vskip 0.2in
\begin{center}
\centerline{\includegraphics[width=\columnwidth]{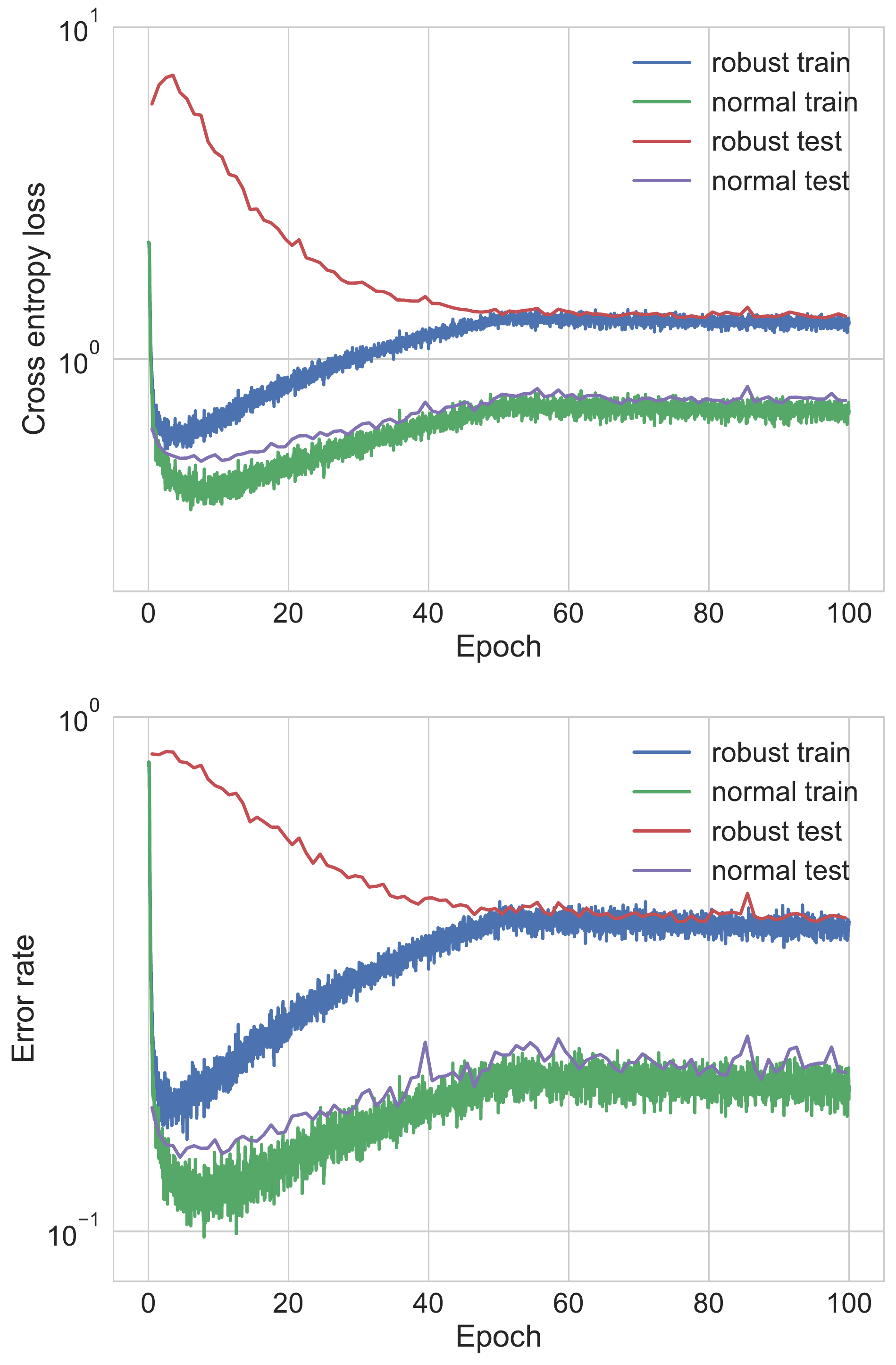}}
\caption{Loss (top) and error rate (bottom) when training a robust convolutional network on the SVHN dataset. The robust test curve is the only curve calculated with $\epsilon=0.01$ throughout; the other 
curves are calculated with the scheduled $\epsilon$ value. }
\label{fig:svhn}
\end{center}
\vskip -0.2in
\end{figure}

\subparagraph{Parameters}
We use the Adam optimizer with the default learning rate 0.001, minibatches of size 20, and trained for 100 epochs. We used an $\epsilon$ schedule which took uniform steps from $\epsilon=0.001$ to $\epsilon=0.01$ over the first 50 epochs. 

\subparagraph{Learning Curves}
Note that the robust testing curve is the only curve calculated with $\epsilon=0.01$ throughout all 100 epochs. 
The robust training curve was computed with the scheduled value of $\epsilon$
at each epoch. We see that all
metrics calculated with the scheduled $\epsilon$ value steadily increase after the first few epochs
until the desired $\epsilon$ is reached. On the other hand, the robust testing metrics for $\epsilon=0.01$ steadily decrease until the desired $\epsilon$ is reached. Since the error rate here increases with
$\epsilon$, it suggests that for the given model capacity, the robust training cannot achieve better 
performance on SVHN, and a larger model is needed. 
\end{document}